%% file: main.tex
\documentclass{article}

\PassOptionsToPackage{numbers,square,sort&compress}{natbib}
\usepackage[preprint]{neurips_2026_vericode}
\usepackage[utf8]{inputenc}
\usepackage[T1]{fontenc}
\usepackage{amsmath,amssymb,amsthm,mathtools}
\usepackage{graphicx}
\usepackage{wrapfig}
\usepackage{booktabs}
\usepackage{microtype}
\usepackage{xcolor}
\usepackage{tikz}
\usepackage{pgfplots}
\pgfplotsset{compat=1.18}
\usepgfplotslibrary{groupplots}
\usetikzlibrary{arrows.meta}
\usepackage{hyperref}
\usepackage{url}
\usepackage{cleveref}

\input{macros}

\input{defns.tex}

\title{Constraint-Aware Training}
\author{%
  Jinwoo Kim\thanks{Work partly done while an intern at Microsoft Research.}\\
  Department of Computer Science and Engineering\\
  University of California-San Diego\\
  La Jolla, CA \\
  \texttt{jik083@ucsd.edu}
}

\begin{document}
\maketitle

\begin{abstract}
\input{sections/abstract}

\end{abstract}

\input{sections/introduction}
\input{sections/setting}
\input{sections/width}
\input{sections/depth}
\input{sections/data}
\input{sections/discussion}

\clearpage
\bibliographystyle{plainnat}
\bibliography{references}

\clearpage
\appendix
\input{sections/appendix}

\end{document}

%% file: macros.tex
\newcommand{\CE}{\mathrm{CE}}

\newcommand{\KL}{\operatorname{KL}}
\newcommand{\E}{\mathbb{E}}

\newtheorem{theorem}{Theorem}

\theoremstyle{definition}
\newtheorem{definition}[theorem]{Definition}
\theoremstyle{remark}

%% file: defns.tex
\makeatletter
\newtheorem*{rep@theorem}{\rep@title}
\newcommand{\newreptheorem}[2]{%
\newenvironment{rep#1}[1]{%
 \def\rep@title{#2 \ref{##1}}%
 \begin{rep@theorem}}%
 {\end{rep@theorem}}}
\makeatother

\newcommand{\model}{L_{\theta}}
\newcommand{\celoss}{\ell_{\mathrm{CE}}}
\newcommand{\celossmask}{\ell_{\mathrm{CE}}^{{S_A}}}
\newcommand{\supportloss}{\ell_{\mathrm{A}}}
\newcommand{\innerloss}{\ell_{\mathrm{CA}}}
\newcommand{\onehot}{\mathbf{1}}
\newcommand{\corpus}{\mathcal{C}}
\newcommand{\innerlossavg}{\hat{\ell}_{\mathrm{CA}}}

\newcommand{\fluctuationloss}{\ell_{\mathrm{var}}}
\newcommand{\meanpredictionloss}{\ell_{\mathrm{mean}}}
\DeclareMathOperator{\inputsize}{size}
\newcommand{\maxinputsize}{s_{\max}}

\newreptheorem{theorem}{Theorem}
\newreptheorem{lemma}{Lemma}

\newcommand{\rone}{(\emph{i})~}
\newcommand{\rtwo}{(\emph{ii})~}
\newcommand{\rthree}{(\emph{iii})~}

\newcommand{\lbbar}{\{\kern-0.5ex|}
\newcommand{\rbbar}{|\kern-0.5ex\}}

\newcommand{\biglbbar}{\left\{\kern-0.5ex\left|}
\newcommand{\bigrbbar}{\right|\kern-0.5ex\right\}}

\usepackage{color}

\newcommand{\Etrue}{{{\mathsf{true}}}}

\newcommand{\calS}{\mathcal{S}}

\newcommand{\Omit}[1]{}

\usepackage[breakable]{tcolorbox}
\newenvironment{mybox}[1][gray!20]{
	\begin{tcolorbox}[   
		breakable,
		left=0pt,
		right=0pt,
		top=0pt,
		bottom=-1pt,
		colback=#1,
		colframe=#1,
		width=\dimexpr\textwidth\relax,
		boxsep=2pt,
		arc=0pt,outer arc=0pt,
		]
	}{
\end{tcolorbox}
}

\newcounter{resq}

\newcommand{\set}[1]{\{{#1}\}}

\definecolor{dgreen}{RGB}{0,128,0}
\definecolor{dred}{RGB}{200,0,0}

%% file: sections/abstract.tex
When generating programs with language models, constrained decoding can apply program analyses to exclude tokens that violate syntax, scope, or typing rules. However, there is a duplication: standard training already teaches the model to suppress the tokens rejected by these analyses. This duplication leads to the question: if we will perform some analysis to filter a set tokens out during inference anyways, can we avoid teaching the model the said analysis altogether during training, and does this externalization lead to more efficient models? This paper defines a general constraint-aware objective satisfying this externalization desideratum and formalizes the benefits of externalization into three concrete theorems about model size and data efficiency. We show, through a controlled synthetic experiment, that the theorems survive training dynamics: constraint-aware training yields lower prediction loss at a matched parameter count and data compared to ordinary cross-entropy training, motivating training objectives that incorporate the analyses used during generation.

%% file: sections/introduction.tex
\section{Introduction}
\label{sec:introduction}

Large language models
(LLMs) are being increasingly used to generate outputs with a formal
syntax and semantics: tool calls, programs, and proofs.
In many cases, such outputs have a notion of \emph{validity}: for example,
only a small fraction of all strings are valid programs.
For formal artifacts, this validity is almost always defined by a
computable analysis (e.g., language membership over a grammar
or typing rules).

Of particular interest for autoregressive (AR) LLMs
is the case where these analyses can become prefix-grounded,
allowing us to soundly filter out next-token selections
depending on whether they would result in a valid output or not.
For example, when completing an expression coming after the prefix
\texttt{"int x = 10; int r = "},
a scope analysis, among candidate variable names, may allow only the in-scope \texttt{x}
while filtering out other undefined names (such as \texttt{y} or \texttt{z}).

This technique, commonly known as \emph{constrained} (or \emph{guided})
\emph{decoding}~\citep{chopchop,gcd,llguidance,vechevtypes}
is well-studied and developed, but there exists
a fundamental mismatch, or inefficiency in its deployment scenario.
In our previous example, a model trained on standard next-token
cross-entropy (CE) will already
place near-zero probability mass on
the tokens corresponding to the variables \texttt{y} or \texttt{z}.
In other words, the model has \emph{already learned}
to separate the valid \texttt{x} from the invalid \texttt{y} and \texttt{z},
even though constrained decoding will provide this information at runtime.
This raises the question: if we will perform a prefix analysis to
filter a certain set of
tokens out during inference anyways, can we \emph{externalize} this analysis
and avoid teaching the model a problem we already know how to solve
symbolically?
More practically, can a model where a prefix analysis is
somehow externalized be more efficient than a model without externalization?

Consider the example again, for a 3-token tokenizer
$\{\texttt{x}, \texttt{y}, \texttt{z}\}$.
Standard CE loss requires the model to predict
the distribution $(\texttt{x}, \texttt{y}, \texttt{z}) = (1, 0, 0)$.
However, if we assume that the model will always be deployed
with the scope analysis at inference,
the space of solutions becomes much freer.
$(1, 0, 0)$, $(0.6, 0.2, 0.2)$, and $(0.1, 0.4, 0.5)$ all become valid predictions, as filtering out the probabilities of
\texttt{y} and \texttt{z} then renormalizing yields the target distribution
$(1, 0, 0)$ for all these distributions;
in general, any distribution $(p_{\texttt{x}},p_{\texttt{y}},p_{\texttt{z}})$
becomes valid as long as $p_{\texttt{x}} \neq 0$.
In other words, the space of solutions where the scope analysis is externalized
is strictly larger than the space of solutions defined by the standard CE loss,
and it is plausible that a model with limited capacity can more easily fit
the larger externalized space than the single CE solution.

Our main observation is that prefix-definable analyses can
be externalized by defining a
\emph{constraint-aware} (CA) loss for training:
at each prefix, apply the analysis,
renormalize over the remaining tokens, and compute cross-entropy against
the target token.
\Cref{fig:loss-scenarios} illustrates how the CA loss is computed:
intuitively, \emph{only} tokens that are admitted by the prefix analysis
receive a signal.

\input{figures/loss-scenarios-figure}

Variants of the CA loss we propose have appeared in
specialized
domains~\citep{reclm,graft};
our contribution is generalizing and
formalizing the effect of the CA loss
precisely as the externalization of a prefix analysis
(\Cref{thm:deployed-likelihood}).
We use this connection to prove three theorems about
the benefits of externalizing analyses:
two \emph{capacity} theorems, which show that
choosing not to externalize an analysis imposes certain
\rone width (\Cref{thm:width}) and \rtwo depth (\Cref{thm:depth})
requirements on the size of the transformer,
and a \emph{data efficiency} theorem (\Cref{thm:variance}) showing
how
constraint-aware training can reduce the amount of data
required to reach a target loss.
We then test, preliminarily,
whether these theoretical benefits carry over to actual
transformer training via controlled small-scale training runs over
a synthetic language.

Overall, our theorems and the results of our experiments
argue for
\emph{constraint-aware training}:
across our experiments, models trained over CA loss
and deployed under constrained decoding
exhibit
consistently lower loss compared to `ordinary'
models trained over CE loss, \emph{both}
when the CE model is deployed normally
(where the CE model
loss is computed over the entire set of tokens)
and under constrained decoding
(where the CE model loss is computed only over the
set of tokens allowed by the analysis).
The empirical advantage persists across training:
CA training attains lower loss even when the theory
no longer provides a guarantee that CA will achieve better loss
(e.g., when the model is wide enough),
providing evidence that
the theoretical benefits of constraint-aware training
transfer to actual training.

%% file: figures/loss-scenarios-figure.tex
\begin{figure}[t]
  \centering
  \resizebox{\linewidth}{!}{\input{figures/loss-scenarios}}
  \caption{
  Standard (CE) and constraint-aware (CA) loss
  compared on three different
  scenarios on each row.
  In each row,
  mask $M$ denotes the set of tokens allowed (marked with $\mathsf{T}$)
  according to the analysis, $t$ the one-hot target, $p$ the model prediction,
  and $p_M$ the model prediction renormalized on tokens allowed by $M$.
  (a) With only the gold \texttt{x} allowed, CA loss is zero.
  (b) When the mask $M$ allows both \texttt{x} and \texttt{y},
   CA penalizes assigning probability
  to the non-gold \texttt{y}.
  (c) Changing the mass on the rejected token (vs. (b))
  without changing the
  $\texttt{x}:\texttt{y}$ ratio changes CE but leaves CA unchanged.}
  \label{fig:loss-scenarios}
\end{figure}

%% file: figures/loss-scenarios.tex
\begin{tikzpicture}[
  x=1cm, y=1cm,
  font=\footnotesize,
  >=Latex,
  masked/.style={text=blue!55!black},
  separator/.style={draw=black!20, line width=0.35pt}
]
  \path[use as bounding box] (0,0) rectangle (14.25,3.55);
  \node[font=\small] at (7.125,3.3)
    {Token order $(\texttt{x},\texttt{y},\texttt{z})$};
  \node at (2.1,2.9) {Prefix mask and gold};
  \node at (5.3,2.9) {Model prediction};
  \node at (9,2.9) {After masking};
  \node at (12.2,2.9) {Scalar losses};
  \draw[separator] (0,2.65) -- (14.25,2.65);

  \begin{scope}[shift={(0,2.2)}]
    \node[anchor=west, font=\footnotesize\bfseries] at (0.1,0) {(a)};
    \node at (2.1,0.18) {$M=(\mathsf{T},\mathsf{F},\mathsf{F})$};
    \node at (2.1,-0.18) {$t=(1,0,0)$};
    \node at (5.3,0) {$p=(0.6,0.2,0.2)$};
    \draw[->, line width=0.55pt] (6.85,0) -- (7.55,0);
    \node[masked] at (9,0) {$p_M=(1,0,0)$};
    \node at (12.2,0.18) {CE: $-\log 0.6 \approx 0.511$};
    \node[masked] at (12.2,-0.18) {CA: $-\log 1 = 0$};
  \end{scope}

  \begin{scope}[shift={(0,1.3)}]
    \node[anchor=west, font=\footnotesize\bfseries] at (0.1,0) {(b)};
    \node at (2.1,0.18) {$M=(\mathsf{T},\mathsf{T},\mathsf{F})$};
    \node at (2.1,-0.18) {$t=(1,0,0)$};
    \node at (5.3,0) {$p=(0.3,0.3,0.4)$};
    \draw[->, line width=0.55pt] (6.85,0) -- (7.55,0);
    \node[masked] at (9,0) {$p_M=(0.5,0.5,0)$};
    \node at (12.2,0.18) {CE: $-\log 0.3 \approx 1.204$};
    \node[masked] at (12.2,-0.18) {CA: $-\log 0.5 \approx 0.693$};
  \end{scope}

  \begin{scope}[shift={(0,0.4)}]
    \node[anchor=west, font=\footnotesize\bfseries] at (0.1,0) {(c)};
    \node at (2.1,0.18) {$M=(\mathsf{T},\mathsf{T},\mathsf{F})$};
    \node at (2.1,-0.18) {$t=(1,0,0)$};
    \node at (5.3,0) {$p=(0.45,0.45,0.1)$};
    \draw[->, line width=0.55pt] (6.85,0) -- (7.55,0);
    \node[masked] at (9,0) {$p_M=(0.5,0.5,0)$};
    \node at (12.2,0.18) {CE: $-\log 0.45 \approx 0.799$};
    \node[masked] at (12.2,-0.18) {CA: $-\log 0.5 \approx 0.693$};
  \end{scope}

  \foreach \separator in {1.75,0.85} {
    \draw[separator] (0,\separator) -- (14.25,\separator);
  }
\end{tikzpicture}

%% file: sections/setting.tex
\section{Externalizing Analyses Using Constraint-Aware Loss}
\label{sec:setting}

Given a prefix (a context) $c$, the goal of an LLM $\model$
is to produce
the distribution $p(\cdot \mid c)$ over tokens $x$ in a tokenizer $V$.
If an analysis $A$ marks some subset $O$ of $V$ as invalid,
one can intuitively understand training $\model$ as trying to satisfy
two separate goals:
\rone sending $p(x \mid c) \rightarrow 0$ for $x \in O$, and
\rtwo modeling $p(x \mid c)$ according to the corpus for $x \not \in O$.
The goal of this section is to establish this decomposition into a precise
mathematical form, then show that a model $\model$ can be trained
on the second objective alone.

\subsection{Constraint-Aware Loss Decomposition}
We start with the definition of a prefix analysis.
\begin{definition}[Prefix Analysis]
    \label{def:prefix-analysis}
    A \emph{prefix analysis} over a finite vocabulary $V$ is a computable
    function $A:V^*\times V\longrightarrow\ \mathsf{Bool}$, where a token
    $t \in V$ is \emph{allowed} at a prefix $c$ iff $A(c, t) = \Etrue$.
    For a prefix $c$,
    we define the \emph{allowed set} $S_A$ as
    $S_A(c)=\{x\in V:A(c,x)=\Etrue\}$,
    the \emph{disallowed set} $O_A$ as
    $O_A(c)=V\setminus S_A(c)$, and the
    \emph{prefix mask}
    $M_A(c,x)=\mathbf 1\{A(c,x)=\Etrue\}$.
\end{definition}
We drop the $A$ and $c$ quantifiers and simply denote the allowed set,
disallowed set, and mask as $S$, $O$, and $M$ when $A$ and $c$ are
clear from context.
In the example from the introduction, the allowed set $S=\set{\texttt{x}}$
while the disallowed set $O=\set{\texttt{y},\texttt{z}}$.
In this paper, we assume analyses $A$ are
\emph{sound} with respect to the training data: for
all prefix-next token pairs $(c, y)$ in the training data, $A(c, y)=\Etrue$.

Given a prefix $c$, let
$p(S\mid c)=\sum_{x\in S}p(x\mid c)$
denote the total probability mass on allowed tokens.
For any fixed next-token $y \in S$, we observe the probability
predicted by the model
$p(y \mid c)$
can be decomposed as
$p(y \mid c) = p(S \mid c) \cdot \frac{p(y \mid c) }{p(S \mid c) }$.
Taking negative log yields a decomposition of the loss:
{\small
\begin{equation}
  \label{eq:loss-decomposition}
 \underbrace{-\log p(y\mid c)}_{\celoss}
 =
 \underbrace{-\log p(S\mid c)}_{\supportloss} +
 \underbrace{-\log\frac{p(y\mid c)}{p(S\mid c)}}_{\innerloss}
\end{equation}
}
In \Cref{eq:loss-decomposition}, $\celoss$ denotes the
standard cross-entropy loss induced by $p(y \mid c)$, which is decomposed
as the sum of two elements:
\rone the \emph{support loss} $\supportloss$ and
\rtwo the \emph{constraint-aware loss} $\innerloss$.
We observe that $\supportloss$ is \emph{precisely} the penalty
induced by assigning non-zero probability to tokens that are disallowed by
the analysis,
while $\innerloss$ is \emph{precisely} the penalty induced by
mispredicting the renormalized distribution $p(\cdot \mid c)$ on the
allowed support $S$.
For example, reconsider row (b) of \Cref{fig:loss-scenarios}.
The losses are
\[
 \celoss =-\log0.3\approx1.204,\qquad
 \innerloss =-\log0.5\approx0.693,\qquad
 \supportloss =-\log0.6\approx0.511.
\]
and $\celoss = \innerloss + \supportloss$.
%
This decomposition carries over to
\emph{gradients}
in the per-token logit space:
$\supportloss$ defines a gradient that pushes mass off of
disallowed tokens, while $\innerloss$ defines a gradient that rebalances
mass amongst allowed tokens, which becomes the main
motivation behind
constraint-aware training.

To see the gradient decomposition, let
$z(x \mid c)$ denote the logit-level prediction for a token $x$ on the
prefix $c$ (i.e., $p(\cdot \mid c) = \operatorname{softmax}(z(\cdot \mid c))$).
Let $p_S$ denote $p$ renormalized on the allowed set $S$, with probability
zero outside $S$.
Let $\onehot_x$ denote the one-hot distribution on $x$.
For a fixed training pair $($prefix $c$, gold token $y)$, differentiating
\Cref{eq:loss-decomposition} over the logit space
for a fixed analysis gives:
{\small
\begin{equation}
  \label{eq:gradient-decomposition}
 \frac{\partial\celoss}{\partial z(x\mid c)}
 =\frac{\partial\supportloss}{\partial z(x\mid c)}
 +\frac{\partial\innerloss}{\partial z(x\mid c)}
 =\underbrace{p(x\mid c)-p_S(x\mid c)}_{\partial\supportloss/\partial z(x\mid c)}
 +\underbrace{p_S(x\mid c)-\onehot_y(x)}_{\partial\innerloss/\partial z(x\mid c)}
 =p(x\mid c)-\onehot_y(x).
\end{equation}
}
(A precise derivation of \Cref{eq:gradient-decomposition} can be found in
\Cref{app:derivations}).
If a token $x$ is \emph{disallowed} by the analysis $A$,
$p_S(x\mid c) = \onehot_y(x) = 0$,
so the entire gradient
comes from the support component
$p(x\mid c) - p_S(x\mid c) > 0$ towards
decreasing the probability of the token $x$.
Conversely, if $x$ is allowed,
the support component $p(x\mid c)-p_S(x\mid c) \leq 0$, so the support
component increases the probability of $x$ while the constraint-aware component rebalances probability amongst allowed tokens.

For example, in row (b) of \Cref{fig:loss-scenarios},
$p(\cdot\mid c)=(0.3,0.3,0.4)$, $p_S(\cdot\mid c)=(0.5,0.5,0)$,
and $\onehot_{\texttt{x}}=(1,0,0)$.
Collecting the entries for \texttt{x}, \texttt{y}, and \texttt{z} gives
\[
 \underbrace{(-0.7,\;0.3,\;0.4)}_{\nabla_z\celoss}
 =\underbrace{(-0.2,\;-0.2,\;0.4)}_{\nabla_z\supportloss}
 +\underbrace{(-0.5,\;0.5,\;0)}_{\nabla_z\innerloss}.
\]
We use $\nabla_z \ell$ as shorthand for
$\bigl(\partial\ell/\partial z(x\mid c)\bigr)_{x\in V}$.
For the disallowed token $\texttt{z}$, only $\nabla_z \supportloss$
gives signal decreasing its probability.
For allowed tokens $\texttt{x}$ and $\texttt{y}$,
$\nabla_z \supportloss$ moves the probability took from \texttt{z}
to \texttt{x} and \texttt{y}
while $\nabla_z \innerloss$ redistributes mass from
\texttt{y} to \texttt{x}.

\subsection{Constraint-Aware Training}

Because an analysis $A$ is independent of the parameters
$\theta$, the chain rule preserves the split when backpropagating through
the model:
\begin{equation}
 \nabla_\theta\celoss
 =\nabla_\theta\supportloss + \nabla_\theta\innerloss
 \label{eq:parameter-gradient-decomposition}
\end{equation}
Constraint-aware training drops the
support gradient $\nabla_\theta \supportloss$ completely
and trains only on $\nabla_\theta \innerloss$.
For a fixed corpus $\corpus$, training exclusively on
$\nabla_\theta \innerloss$ has the effect of
fitting the model's probabilities \emph{after masking and renormalization}
to how often each token follows each prefix in $\corpus$.

Precisely speaking, let
$\corpus=((c_i,y_i))_{i=1}^N$ denote a multiset of
prefix--next-token training pairs.
Let $n_\corpus(c)$ denote the number
of pairs in $\corpus$ with prefix $c$,
and $\widehat p_\corpus(x\mid c)$ the fraction of
these pairs whose next token is $x$.
Constraint-aware training upsates the parameters $\theta$
mimimize the average CA loss
$\innerlossavg(\theta; \corpus)$
over the training corpus $\corpus$:
{\small
\begin{equation}
 \innerlossavg(\theta;\corpus)
 =\frac1N\sum_{i=1}^N\innerloss(\theta;c_i,y_i).
 \label{eq:ca-empirical-loss}
\end{equation}
}
\Cref{thm:deployed-likelihood} states the correctness of this
training objective.

\begin{theorem}[Correctness of Constraint-Aware Training]
\label{thm:deployed-likelihood}
Let $\corpus$ be a training-pair multiset and $A$ an analysis.
Let
$H_\corpus$ denote the average next-token entropy of $\corpus$:
$H_\corpus =\frac1N\sum_{c:n_\corpus(c)>0}n_\corpus(c)H
(\widehat p_\corpus (\cdot\mid c))$
and $p_{S_A(c), \theta}(\cdot | c)$ the next-token distribution
$p_\theta (\cdot | c)$ renormalized on the allowed set $S_A(c)$.
Then
{\small
\begin{equation}
 \innerlossavg(\theta;\corpus)-H_\corpus
 =\frac1N\sum_{c:n_\corpus(c)>0}n_\corpus(c)\,
 \KL\!\left(\widehat p_\corpus(\cdot\mid c)\,\middle\|\,
 p_{S_A(c),\theta}(\cdot\mid c)\right).
 \label{eq:ca-empirical-correctness}
\end{equation}
}
Because the entropy term $H_\corpus$ is fixed,
minimizing $\innerlossavg(\theta;\corpus)$ also minimizes the right-hand side
term, the weighted average Kullback--Leibler (KL) divergence between
\rone $\widehat p_\corpus(\cdot\mid c)$,
the next-token distribution defined by $\corpus$, and
\rtwo $p_{S_A(c), \theta}(\cdot | c)$,
the distribution predicted by $\theta$ when deployed with the analysis $A$.
\end{theorem}

\Cref{thm:deployed-likelihood} proves that constraint-aware training
minimizes KL divergence with the training set:
the same correctness statement as standard training, \emph{conditioned}
on the renormalized distributions obtained by applying the analysis.
For example, if the training-pair set $\corpus$ contains
a prefix is followed by \texttt{x} four times and
\texttt{y} once, and the analysis allows only those two tokens,
both the distributions
$(\texttt{x}, \texttt{y}, \texttt{z}) = (0.8, 0.2, 0.0)$ and
$(0.4, 0.1, 0.5)$ satisfy the
$\innerlossavg(\theta;\corpus)$ objective
(unlike
standard training, which would force the model towards $(0.8, 0.2, 0.0)$).

\Cref{thm:deployed-likelihood} defines, mathematically, how
constraint-aware training has \emph{externalized} the analysis:
it effectively ignores tokens that would be filtered out by the analysis
when minimizing KL divergence against the distribution of the training
set.
During inference, the analysis filters these tokens out anyways;
constraint-aware training thus focuses solely on learning the
relative distribution between \texttt{x} and \texttt{y}.
\Cref{app:ca-correctness,app:ca-strings}
provide a proof of \Cref{thm:deployed-likelihood} and a connection
to sequence likelihood over $\corpus$.

%% file: sections/width.tex
\section{The Benefits of Externalization: Width}
\label{sec:width-theory}

Now that we have established how constraint-aware training
externalizes an analysis, we proceed to show how externalization
brings concrete benefits: starting by identifying an
analysis that enforces a \emph{lower bound} on the \emph{width}
of a transformer, and showing that externalization can alleviate
this bound.

We illustrate the problematic analysis with an example.
Consider a language with five candidate variable names
$\{\texttt{x1},\ldots,\texttt{x5}\}$, of which exactly two are in scope.
Different prefixes can have any of the ten possible pairs in scope:
for example,
at the prefix \texttt{int x1 = 0; int x2 = 0; int r =},
a scope analysis allows only $S=\{\texttt{x1},\texttt{x2}\}$.
Suppose the corpus uses these two names equally often, so the target
distribution over the five candidates
at this point is
$(0.5,0.5,0,0,0)$.

Let us assume that the transformer contains a
\emph{readout} layer
$L_R$ of size $|V| \times w$ that
consumes a hidden state of width $w$
and converts it
into logits over the vocabulary $V$.
Let us assume, in our example, that we are working with a trunk
that produces a hidden state of
width $w=1$, and thus the readout receives a single number $h$ from the trunk.
Under the CA objective, a possible solution for $L_R$
is $L_R = (0 \ 1 \ 2 \ 3 \ 4)^\top$: $h=0$ yields a
base probability distribution of
$(0.2,0.2,0.2,0.2,0.2)$, which becomes the target
$(0.5,0.5,0,0,0)$ after masking and renormalization.
Assuming the trunk is sufficiently capable,
$L_R = (0 \ 1 \ 2 \ 3 \ 4)^\top$ is a \emph{concrete solution}
for the optimization
problem the CA objective tries to solve:
for any allowed pair of tokens and their probability ratio,
there exists an output $h$ from the trunk such that
$\operatorname{softmax}(L_R h)$ renormalized obtains the target ratio.

On the other hand, without the analysis filter,
\emph{there does not exist} a readout $L_R$ that is sufficient to obtain
any desired probability ratio while assigning arbitrarily little
probability to the disallowed tokens:
multiplying the five readout weights in $L_R$ by $h$ can only preserve their
ordering ($h>0$), reverse it ($h<0$), or make all tokens equally likely
($h=0$), but not make each of the ten possible allowed pairs
outrank the other three tokens.
In fact, the width $w$ must be at least $4$ for the CE objective to
model any pair-live distribution in our example
(\Cref{app:width} and \Cref{thm:width}):
internalizing the analysis enforces a
\emph{minimal width} on the transformer, larger than when the
analysis is externalized.
\Cref{thm:width} formalizes this difference in required width.

\begin{theorem}[Minimum Width for Analysis]
\label{thm:width}
Let $A$ be a prefix analysis that allows between one and $m$ tokens
at each prefix, where $m\ge2$. Suppose that, as the prefix varies,
its allowed set can be any $m$-element subset of a pool of at least
$2m+1$ tokens. Then
the minimum hidden-state widths permitting a capable
readout are
\begin{equation}
 w^{\mathrm{cap}}_{\mathrm{CA}}=m-1,\qquad
 w^{\mathrm{cap}}_{\CE}=2m.
 \label{eq:width-gap}
\end{equation}
\end{theorem}
In \Cref{thm:width}, we call a readout $L_R$
\emph{capable} for any target distribution on every allowed set,
there exists some input of hidden states $h$ such that
$L_R$ achieves arbitrarily small excess loss.
A precise definition is given in
\Cref{app:width}, along with a more precise theorem statement
and a proof.

\Cref{thm:width} establishes a \emph{necessity}: a capable
readout is required to support an analysis $A$ which can
allow any $m$-element subset of tokens
with arbitrary
precision, thus the hidden state must be at least as wide as
$m-1$ for CA and $2m$ for CE.\footnote{
  Even if the lower width bound in \Cref{thm:width} is
  satisfied,
  achieving arbitrary precision in practice
  is a training question that requires
  the model to find the capable readout
  (and a matching capable trunk) through training.
}
Analyses
that allow $m$-subsets of tokens
are relatively common in programming:
as the prefix varies,
scope or type analyses
can allow any $m$-element subset of a pool of at least
$2m+1$ identifiers (e.g., constants, variable names).\footnote{
  The mismatch between program identifiers and tokens is discussed
  in \Cref{app:tokenization}.
}

\paragraph{Experimental setup.}
We test whether the predicted difference in required
width transfers to an actual advantage in model size
by training a transformer on a synthetic corpus modeling a
scope analysis.
Each statement in the corpus is of the form
$\texttt{(scope x\_i1 \ldots x\_i12 f\_j y)}$,
where
\rone the variables $\texttt{x\_i1 \ldots x\_i12}$ denote the
12 variables in-scope
(drawn from a universe of 61 variable names
\texttt{x\_1} \ldots \texttt{x\_61}
),
\rtwo $\texttt{y}$ is one of the 12 in-scope variables,
and
\rthree $\texttt{f\_j}$ is one of 12 functions, each of which
defines a probability distribution over the variables:
e.g., $\texttt{f\_1}$ might assign
a distribution $(0.5, 0.5, 0, \ldots 0)$ to
$(\texttt{x\_1} \ldots \texttt{x\_61})$.

We train models
with four layers of internal dimension 64, followed by
a learned projection to $w$ dimensions, while varying $w$.
The model is trained to predict the token at $\texttt{y}$ given
the scope and the function as a prefix
(details in \Cref{app:width-experiments}).
The allowed set $S_A(c)$ of the analysis is the set of 12 names
that are in scope;
model trained on CA loss must learn the
probability distribution
that the function name \texttt{f\_j} enforces,
while a model trained on CE loss must learn to suppress the 49 out-of-scope
variables in addition.
\Cref{thm:width} states
that $w$ must be at least $11$ for CA and $24$ for CE to achieve
arbitrary precision.
(Whether training finds an accurate model
(both the readout and the trunk)
at these widths is a separate empirical question.)
We measure excess loss on a held-out dataset after training:
$\innerloss$ for the CA objective,
$\celoss$ for the CE objective, and also
$\celossmask$, the loss of the CE model limited to the allowed set $S_A(c)$.

\input{figures/width-results-figure}

\paragraph{Results.}
\Cref{fig:width-results} reports excess loss
(loss minus the target distribution
entropy)
on 480,000 held-out queries.
As expected from \Cref{thm:width}, $\innerloss$ drops sharply when the
width reaches $11$: from 0.01 at $w=10$ to 0.0004 at $w=11$.
Beyond $w > 11$ loss stays near 0.0004--0.0005, with little benefit from
additional width;
CA training finds a highly accurate model as soon as the
readout has sufficient width to express the variable-usage preferences.

On the other hand, standard CE shows significantly higher loss.
Excess loss $\celoss$ stays high ($> 0.8$) when $w < 24$,
and although $\celoss$
does drop sharply starting from $w=24$, it remains over
0.05 across all widths.
Although \Cref{thm:width} states that an optimal readout
(that is, a readout capable of arbitrary precision given a perfect
trunk)
\emph{exists}
for CE at $w=24$,
the CE model failed to learn it;
$w=32$ was the smallest width
displaying capability across sampled queries
(see \Cref{app:width-experiments} for details).
Even then, $\celoss$ was 0.1812, indicating that the
readout-trunk pair failed to reach an optimal state.


Both CA and CE models retained residual error after training;
$\innerloss$ settles near
0.0005, while $\celoss$ reached 0.1812 at $w=32$ and 0.0537 at
$w=64$, even though its learned readouts were empirically capable.
$\celoss$ continues to improve between these widths, but
remains far from $\innerloss$.
Overall, constraint-aware training displayed substantial
benefits: $\innerloss$ was smaller than both $\celoss$
and $\celossmask$
across all widths, often by orders of magnitude.
The result at $w=11$ shows that CA training can achieve much lower
excess loss even at smaller model sizes.
Externalizing scope during
training yielded a much closer fit to the variable-usage
distribution throughout this sweep, including at widths well above the
capability thresholds of both objectives.
\Cref{app:width-experiments} illustrates additional
controls using a fixed scope and a longer training schedule.

%% file: figures/width-results-figure.tex
\begin{wrapfigure}{R}{0.49\textwidth}
 \centering
 \input{figures/width-results}
 \caption{Mean losses $\innerloss$, $\celossmask$,
 and $\celoss$, each reported in excess of the target distribution's entropy.
 All values are nats per answer.}
 \label{fig:width-results}
\end{wrapfigure}

%% file: figures/width-results.tex
\definecolor{CAblue}{HTML}{2376A8}
\definecolor{CEred}{HTML}{B93C48}
\definecolor{CEorange}{HTML}{CF841C}
\begin{tikzpicture}
\begin{axis}[width=\linewidth,height=4.7cm,xmode=log,ymode=log,
xlabel={Readout width},ylabel={Excess loss (nats)},
xtick={8,11,24,64},xticklabels={8,11,24,64},ymin=0.0001,ymax=2,
legend columns=3,
legend style={at={(0.97,0.38)},anchor=east,font=\scriptsize,
draw=gray!30,fill=white,inner sep=1pt,nodes={inner sep=1pt},
column sep=2pt},legend cell align=left,legend image post style={xscale=0.65},
tick label style={font=\scriptsize},label style={font=\small},
grid=major,grid style={gray!20},/tikz/mark size=1.5pt,scaled ticks=false,minor tick num=0]
\addplot[CAblue,mark=o,thick] coordinates {(8,0.0341) (10,0.01) (11,0.0004) (12,0.0004) (16,0.0005) (23,0.0005) (24,0.0005) (32,0.0005) (64,0.0005)};
\addplot[CEred,mark=square*,thick] coordinates {(8,0.2589) (10,0.257) (11,0.2623) (12,0.2616) (16,0.2295) (23,0.1622) (24,0.1538) (32,0.0733) (64,0.0271)};
\addplot[CEorange,mark=triangle*,dashed,thick] coordinates {(8,1.2369) (10,1.1234) (11,1.0698) (12,1.0169) (16,0.8099) (23,0.4867) (24,0.4505) (32,0.1812) (64,0.0537)};
\legend{$\innerloss$,$\celossmask$,$\celoss$}
\end{axis}
\end{tikzpicture}

%% file: sections/depth.tex
\section{The Benefits of Externalization: Depth}
\label{sec:depth-theory}

We next consider analyses that
cannot be computed by a single layer of weights and
impose a requirement on the
\emph{depth} of a
transformer~\cite{logregulartrans,cfgtrans,maskedtransformers}:
externalizing these analyses can
reduce the depth needed to model the target distribution.
Again, we illustrate with an example.
Consider a language with five disjoint groups of
four variable names,
$\{\texttt{x1},\ldots,\texttt{x4}\}$ through
$\{\texttt{x17},\ldots,\texttt{x20}\}$.
At a prefix \texttt{(if F : tag0},
where $\texttt{F}$ is an arithmetic formula,
let an analysis evaluate $\texttt{F}\!\! \mod 5$ and
allow only the group corresponding to its value:
e.g., if $\texttt{F}=(2+4)\times3$, then
$\texttt{F}\!\! \mod 5=3$ and the allowed set
$S_A(c)=\{\texttt{x13},\ldots,\texttt{x16}\}$.

Suppose \texttt{tag0} specifies a target distribution
$(0.4,0.3,0.2,0.1)$ over the allowed four names.
The CA objective can simply learn to
associate the tag name $\texttt{tag0}$
with the relative
weights $(4,3,2,1)$ to the four names in \emph{every} group:
then,
whichever group the analysis allows,
renormalizing on that group
yields the target distribution $(0.4,0.3,0.2,0.1)$.
On the other hand, a model trained on CE loss must identify
\emph{which} group to assign its probability mass to,
which requires, in essence, the model to be able to
compute the formula $\texttt{F}$.
If evaluating $\texttt{F}$ requires more layers than learning the
tag-dependent probabilities, then internalizing the analysis imposes
an additional depth requirement, similar to the width requirement in
\Cref{thm:width}.
\Cref{thm:depth} formalizes this intuition.

\begin{theorem}[Depth required to internalize an analysis]
\label{thm:depth}
Let $\corpus$ be a training corpus and $A$ be an analysis.
Let $d_A(\alpha)$ be the minimum depth needed to identify the
allowed set with average error probability at most
$\alpha$, over the distribution of prefixes
$c$ defined by the corpus $\corpus$, and
$\pi_c$ the target next-token distribution at prefix $c$
modeled by $\corpus$.

Then for an error bound $\epsilon$ where $0 < \epsilon < \log 2$,
for a model to achieve expected KL divergence
 $\E_c\!\left[\KL\!\left(\pi_c\,\middle\|\,p_\theta(\cdot\mid c)\right)\right]
 \le\epsilon$,
it must have depth at
least $L\ge d_A\!({\epsilon}/{\log2})$.
\end{theorem}


In \Cref{thm:depth}, we use `average probability at most $\alpha$'
to mean that the transformer is allowed to compute the wrong
allowed set for at most an $\alpha$-fraction of prefixes in
$\corpus$;
a precise definition is given in \Cref{app:depth-definition}.
We note that for a fixed $\alpha$, many common analyses
force $d_A(\alpha)$ to \emph{increase} as the corpus contains
\emph{samples of longer size}:
for example, formal language
recognition~\cite{cfgtrans,maskedtransformers,logregulartrans}
often requires the depth of the transformer to
grow with input length.

This observation yields the practical reading of \Cref{thm:depth}:
an CE-trained transformer must grow in depth in order to
support expected loss $\leq \epsilon$ on inputs of longer size.
Even if an anaylsis requires only constant depth
(e.g., the scope analysis from \Cref{sec:depth-theory}),
CE training enforces this depth as a lower bound;
constraint-aware training is free from this lower bound
or the scaling requirement.\footnote{
  We do note that modeling the
  inner distribution of allowed variables
  may require transformers of depth larger than the
  lower bounds imposed by \Cref{thm:depth}.
  How the depth requirements of the inner distribution interact
  in this case
  with the bounds in \Cref{thm:depth} is left for future work.
}
A more precise version of \Cref{thm:depth} and a proof
can be found in \Cref{app:depth-definition} and \Cref{app:depth-hardness}.

\paragraph{Experimental setup.}
We test the depth separation on a synthetic corpus
following the example at the start of this section:
each statement is of the form
$\texttt{(if F : tag\_i y)}$,
\texttt{F} is an arithmetic formula as described, and the
goal of the model is to predict the token at the position
\texttt{y}.
The analysis evaluates \texttt{F} and returns the set of
four allowed names, so the CA objective only needs to learn
the probability distribution specified
by \texttt{tag\_i} across the allowed variable names;
the CE objective must learn to evaluate \texttt{F} in addition,
in order to drive the probability of the
disallowed variables to zero.

We train models of width 128
while varying depth over 2, 4, and 8 layers,
with two initialization seeds at each depth.
Training details are in \Cref{app:depth-experiments}.
We measure the excess losses associated with
$\innerloss$, $\celoss$,
and $\celossmask$, as in \Cref{sec:width-theory},
and also the accuracy of the CE model in
computing the correct allowed set
as a function of complexity (number of operators)
in evaluating \texttt{F}.

\input{figures/depth-results-figure}

\paragraph{Results.}
\Cref{fig:depth-results} summarizes the results
of the experiment, evaluated on a held out set of
$\approx 550$k queries: the left graph illustrates that
$\innerloss$, the loss of the CA-trained model,
outperforms both $\celoss$ (the total loss of the CE-trained
model)
and $\celossmask$ (the loss of the CE model
measured only on the allowed tokens).
In particular, CA loss stays relatively constant at around
$2.7 \times 10^{-5}$ nats as depth increases, but
CE loss improves significantly with depth:
from $1.1$ nats at depth $2$ to $0.53$ at depth $8$,
illustrating that the model is using (and requires)
the additional layers for evaluating \texttt{F} (the analysis).

The right of \Cref{fig:depth-results} illustrates that
the complexity of \texttt{F} is directly correlated with
the loss of the CE model: for example,
the CE model shows low excess loss on formulas
$2-3$ operators, even at depth $2$.
However, the loss starts increasing sharply as the number of
operators increases, eventually plateauing at about 1.1 nats,
at which point the CE model fails to make a prediction
about the allowed set at all and instead just
predicts a uniform distribution: the experiment thus illustrates
that the CE model is \emph{actually learning} an
input-size dependent circuit for evaluating \texttt{F} via
training, necessitating increasingly larger depths to support
inputs of larger size.
Constraint-aware training is free of this burden and achieves
a much lower constant excess loss compared to CE models at
any depth.
\Cref{app:depth-experiments} gives more detailed
experimental results.

%% file: figures/depth-results-figure.tex
\begin{figure}[t]
 \centering
 \input{figures/depth-results}
 \caption{
Left: excess losses obtained by training on the depth corpus,
associated with
$\celoss$, $\celossmask$ (from CE training),
and $\innerloss$ (from CA training).
All losses are nats per answer.
Right: CE excess loss by formula size for 2, 4, and 8 depth
configurations.}
 \label{fig:depth-results}
\end{figure}
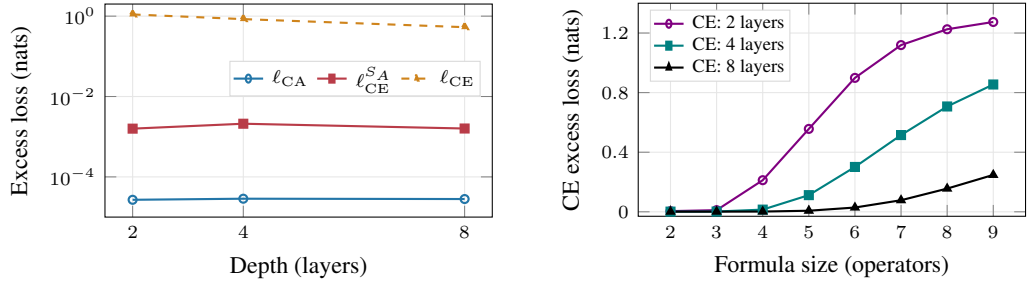

%% file: figures/depth-results.tex
\definecolor{CAblue}{HTML}{2376A8}
\definecolor{CEred}{HTML}{B93C48}
\definecolor{CEorange}{HTML}{CF841C}
\begin{minipage}[c]{0.48\linewidth}
\centering
\begin{tikzpicture}
\begin{axis}[width=\linewidth,height=4.4cm,ymode=log,
xlabel={Depth (layers)},ylabel={Excess loss (nats)},
xtick={2,4,8},xmin=1.5,xmax=8.5,ymin=0.00001,ymax=2,
legend columns=3,
legend style={at={(0.97,0.65)},anchor=east,font=\scriptsize,
draw=gray!30,fill=white,inner sep=1pt,nodes={inner sep=1pt},
column sep=2pt},legend cell align=left,legend image post style={xscale=0.65},
tick label style={font=\scriptsize},label style={font=\small},
grid=major,grid style={gray!20},/tikz/mark size=1.5pt,scaled ticks=false,minor tick num=0]
\addplot[CAblue,mark=o,thick] coordinates {(2,2.69534887262e-05) (4,2.86711919664e-05) (8,2.81392968214e-05) };
\addplot[CEred,mark=square*,thick] coordinates {(2,0.00158734023443) (4,0.0021042004349) (8,0.0015998564507) };
\addplot[CEorange,mark=triangle*,dashed,thick] coordinates {(2,1.10213078429) (4,0.843464168564) (8,0.531236515702) };
\legend{$\innerloss$,$\celossmask$,$\celoss$}
\end{axis}
\end{tikzpicture}
\end{minipage}\hfill
\begin{minipage}[c]{0.48\linewidth}
\centering
\begin{tikzpicture}
\begin{axis}[width=\linewidth,height=4.4cm,
xlabel={Formula size (operators)},ylabel={CE excess loss (nats)},
xtick={2,3,4,5,6,7,8,9},ymin=-0.03,ymax=1.4,ytick={0,0.4,0.8,1.2},
legend style={at={(0.03,0.97)},anchor=north west,font=\scriptsize,
draw=gray!30,fill=white,inner sep=1pt,nodes={inner sep=1pt},
column sep=2pt},legend cell align=left,legend image post style={xscale=0.65},
tick label style={font=\scriptsize},label style={font=\small},
grid=major,grid style={gray!20},/tikz/mark size=1.5pt,scaled ticks=false,minor tick num=0]
\addplot[violet,mark=o,thick] coordinates {(2,0.00456895796972) (3,0.0110993133651) (4,0.212217822884) (5,0.55652747056) (6,0.898944479818) (7,1.11907861316) (8,1.22528723121) (9,1.27417921354)};
\addplot[teal,mark=square*,thick] coordinates {(2,0.00129595777586) (3,0.00190288281288) (4,0.0146911028415) (5,0.111932646145) (6,0.301056899311) (7,0.51513469933) (8,0.70703266545) (9,0.854558666922)};
\addplot[black,mark=triangle*,thick] coordinates {(2,0.000973229148164) (3,0.00114176865211) (4,0.00173309756088) (5,0.00722989328801) (6,0.0287168276946) (7,0.0777289354121) (8,0.15578149912) (9,0.247642898258)};
\legend{CE: 2 layers,CE: 4 layers,CE: 8 layers}
\end{axis}
\end{tikzpicture}
\end{minipage}

%% file: sections/data.tex
\section{The Benefits of Externalization: Data Efficiency}
\label{sec:data}

Having established that analyses
require lower bounds on the size of transformers,
we now show that externalizing an analysis can
reduce the amount of \emph{data} required to achieve a certain
error.
The main contributor is \emph{variance} across training corpora:
different finite corpora expose the model to different prefixes.
Consequently, even models that fit their training targets exactly
can make different predictions at the same unseen prefix.
Under certain assumptions about training, \Cref{thm:variance}
shows that this variation can lead to higher expected loss
against the true distribution $\pi_{\text{true}}$ for CE-trained
models than for CA-trained models.

\begin{samepage}
\begin{theorem}[Data efficiency]
\label{thm:variance}
Let $R_{\mathrm{CA}}(n)$ and $R_{\CE}(n)$ be the expected KL divergences
from $\pi_{\text{true}}(\cdot\mid c)$ to the two models' masked
predictions after training on a corpus of $n$ prefixes,
averaged over corpora and held-out prefixes.
Under the minimum-change fitting and conditional-centering assumptions
in \Cref{app:variance-fitting,app:variance-loss},
\[
 R_{\mathrm{CA}}(n)\le R_{\CE}(n).
\]
\end{theorem}
\end{samepage}

We observe that \Cref{thm:variance} does \emph{not}
imply a loss floor: both CA and CE distributions can be
centered
around $\pi_{\text{true}}(\cdot\mid c)$, but have a positive
expected KL divergence (i.e., loss).
At a high level, the proof of the theorem shows that
CE training results in a larger change of parameters compared to
CA training, which translates under
the conditional-centering assumption to
a higher degree of variance.
The proof, and the assumptions on which this proof hold,
are presented in
\Cref{app:variance-loss,app:sample-requirement}.

\paragraph{Experimental setup.}
In addition to stating that CE training requires
equal or more samples than CA training to achieve the same loss,
\Cref{thm:variance} also names the source of the
high loss displayed by CE: variance.
We thus center our experiment on two aspects:
\rone showing that CA training is indeed more data-efficient,
and
\rtwo varying the corpus used for training while keeping other
parameters fixed, to measure the variance in loss induced by
variation of the training corpus.

We use the same family of statements as the scope analysis
in \Cref{sec:depth-theory}, and train transformers
of depth 4 and internal width 128 by varying
\rone
the corpus over 2.4M, 7.2M, 24M, 36M training pairs,
and \rtwo draw 16 independent corpora at each size.
We fix all other hyperparameters and allocate the same training
compute by revisiting examples in smaller training pools to
isolate the effects of corpus size and variation.
We report excess loss $\celoss$, $\celossmask$ (for CE),
and $\innerloss$ (for CA)
as in \Cref{sec:width-theory,sec:depth-theory}.
For variance, we report
\rone $\meanpredictionloss$,
the loss of the mean prediction generated by
averaging the predictions of all 16
models trained on different corpuses (for a fixed size)
against the true distribution\footnote{
  Our corpus is generated via a generator modeling
  $\pi_{\text{true}}$, the target distribution defined in
  \Cref{sec:width-theory,app:experiments}.
}
and
\rtwo $\fluctuationloss$, the average individual loss
of each model minus $\meanpredictionloss$:
{\small
\begin{equation}
 \meanpredictionloss
 =\E_c\!\left[\KL\!\left(
   \pi_{\text{true}}(\cdot\mid c)\,\middle\|\,
   \bar q_A(\cdot\mid c)\right)\right]
  \qquad
 \fluctuationloss
 =\E_c\!\left[\frac{1}{16}\sum_{i=1}^{16}\KL\!\left(
   \pi_{\text{true}}(\cdot\mid c)\,\middle\|\,
   q_{i,A}(\cdot\mid c)\right)\right]-\meanpredictionloss\ge0.
 \label{eq:prediction-variance}
\end{equation}
}
Where $q_{i, A}$ is the
\emph{allowed-set}
distribution of the model trained on the
corpus $\corpus_i$ and
$\bar q=\frac{1}{16}\sum_{i=1}^{16}q_{i, A}$.
Intuitively, $\fluctuationloss$ measures variance: how much of an
individual model's loss can be attributed to corpus variation.
Training details and controls are in
\Cref{app:data-experiments}.

\input{figures/data-results-figure}

\paragraph{Results.}
\Cref{fig:data-results} summarizes the CA vs. CE
loss investigation (left) and the $\meanpredictionloss$ /
$\fluctuationloss$ metrics (right).
Again, CA achieves much lower loss compared to CE
across all document sizes.
In particular, the measured $\fluctuationloss$
metrics (right)
indicate that variance for CA has flattened
by 7.2M training pairs,
while CE only shows a decreasing trend
in variance reduction from 24M to 36M pairs.
Overall, we observe that
CA achieves lower excess loss with substantially
fewer training samples,
and its predictions fluctuate less across independently
sampled corpora.
(We are unsure of the cause of the
$\approx 7 \times 10^{-5}$ floor observed in CA training, but
conjecture that it is noise related with training or
nondeterminism of the training kernel dominating corpus
variance.)

The right table of \Cref{fig:data-results} illustrates that
$\fluctuationloss$ is much higher compared to
$\meanpredictionloss$
(roughly an order of magnitude higher),
indicating that variance is a
significant contributor to expected excess loss.
These results indicate that variance reduction is an important
element for minimizing excess loss during training, and thus
that the data efficiency advantage of
constraint-aware training manifest as practical benefits.
We do note that the higher $\meanpredictionloss$ of
CE is not explained by our experiments, but the results still
illustrate that CA displays significant benefits in
variance reduction, reaching a floor with 7.2M training samples.

%% file: figures/data-results-figure.tex
\begin{figure}[!t]
 \centering
 \input{figures/data-results}
 \caption{
 Excess loss curves against corpus size (left), and
 loss of the mean prediction across all 16 models
 ($\meanpredictionloss$)
 and the average
 variance component of individual losses
 ($\fluctuationloss$) across corpus sizes (right).
 Units are in nats.
 CA achieves a significantly lower excess
 at all corpus sizes.
 For both CE and CA, $\fluctuationloss$ is much higher then
 $\meanpredictionloss$,
 indicating that variance from the training corpus
 is a significant source of loss.
}
 \label{fig:data-results}
\end{figure}
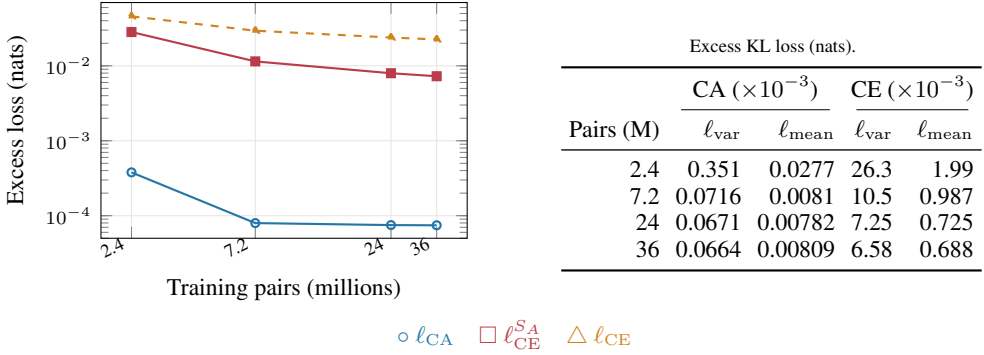

%% file: figures/data-results.tex
\definecolor{CAblue}{HTML}{2376A8}
\definecolor{CEred}{HTML}{B93C48}
\definecolor{CEorange}{HTML}{CF841C}
\begin{minipage}[c]{0.46\linewidth}
\centering
\begin{tikzpicture}
\begin{axis}[width=\linewidth,height=4.7cm,xmode=log,ymode=log,
xlabel={Training pairs (millions)},ylabel={Excess loss (nats)},
xtick={2.4,7.2,24,36},xticklabels={2.4,7.2,24,36},
x tick label style={rotate=25,anchor=east,font=\scriptsize},
ymin=0.00005,ymax=0.07,
tick label style={font=\scriptsize},label style={font=\small},
grid=major,grid style={gray!20},/tikz/mark size=1.5pt,scaled ticks=false,minor tick num=0]
\addplot[CAblue,mark=o,thick] coordinates {(2.4,0.000378766147667) (7.2,7.97291553176e-05) (24,7.49586320827e-05) (36,7.44578403431e-05)};
\addplot[CEred,mark=square*,thick] coordinates {(2.4,0.0282759429505) (7.2,0.0114669797521) (24,0.00797261917667) (36,0.00726834170761)};
\addplot[CEorange,mark=triangle*,dashed,thick] coordinates {(2.4,0.0457657100835) (7.2,0.0294355066983) (24,0.0238750375746) (36,0.0224873519194)};
\end{axis}
\end{tikzpicture}
\end{minipage}\hfill
\begin{minipage}[c]{0.52\linewidth}
\centering\small
{\scriptsize Excess KL loss (nats).\par}
\smallskip
\setlength{\tabcolsep}{3pt}
\begin{tabular}{rrrrr}
\toprule
 & \multicolumn{2}{c}{CA ($\times10^{-3}$)} & \multicolumn{2}{c}{CE ($\times10^{-3}$)}\\
\cmidrule(lr){2-3}\cmidrule(lr){4-5}
Pairs (M) & $\fluctuationloss$ & $\meanpredictionloss$ & $\fluctuationloss$ & $\meanpredictionloss$\\
\midrule
2.4 & 0.351 & 0.0277 & 26.3 & 1.99\\
7.2 & 0.0716 & 0.0081 & 10.5 & 0.987\\
24 & 0.0671 & 0.00782 & 7.25 & 0.725\\
36 & 0.0664 & 0.00809 & 6.58 & 0.688\\
\bottomrule
\end{tabular}
\end{minipage}
\par\smallskip
{\small\textcolor{CAblue}{$\circ$ $\innerloss$}\quad
\textcolor{CEred}{$\square$ $\celossmask$}\quad
\textcolor{CEorange}{$\triangle$ $\celoss$}}

%% file: sections/discussion.tex
\section{Discussion and Related Work}
\label{sec:related-work}

\paragraph{Gradient Decomposition and Semantic Objectives.}
We observe that the loss and gradient decomposition in
\Cref{sec:setting} can be extended to multiple analyses:
for example,
a parser can reject syntactically invalid tokens,
a scope analysis can then remove unbound names, and a
type analysis could then further restrict remaining tokens.
If the analyses can be \emph{nested}: that is,
denoting the allowed set of the $j$-th analysis $A_j$ as
$S_j$, we observe that the loss decomposition
can be chained:
{\small
\[
 \ell_j=-\log\frac{p(S_j\mid c)}{p(S_{j-1}\mid c)}.
\]
}
Where $\ell_j$ quantifies how much of the remaining
probability mass (after $A_{j-1}$) the analysis $A_j$
removes.
Summing these nested losses one can obtain
$\celoss=\innerloss+\sum_{j=1}^k\ell_j$, where
$\innerloss$ is the final allowed set $S_k$.
The interesting part is that because we have isolated the
contribution of each analysis, we may assign a separate
weight to it:
$\ell'=\innerloss+\sum_{j=1}^k \lambda_j \ell_j$,
which defines a \emph{semantically modified training objective}
where each component contributed by an analysis $A_j$
is weighed differently.
Take the gradient of the modified objective to obtain
{\small
\begin{equation}
 \nabla_\theta\!\left(\innerloss+\sum_{j=1}^k\lambda_j\ell_j\right)
 =\nabla_\theta\innerloss+\sum_{j=1}^k\lambda_j\nabla_\theta\ell_j.
 \label{eq:semantic-objective}
\end{equation}
}
\Cref{eq:semantic-objective} illustrates how the
coefficients survive into the gradients over the parameter space
$\nabla_\theta \ell_j$
generated by each analysis, thus showing how the
coefficients can amplify or dampen signal.
This paper studies the specific case where
all $\lambda_i = 0$.
Whether different configurations of $\lambda_i$ lead to
practically relevant training mechanisms is a very
interesting direction of future work.

\paragraph{Constrained and Guided Decoding.}
There is a plethora of work on constrained / guided
decoding, much of which are focused on developing
different efficient prefix analyses for computing
the allowed set~\cite{gcd,chopchop,vechevtypes,llguidance}.
This paper discusses how to apply previous work on
developing prefix analyses towards training.

Of particular interest are works which
focus on making constrained decoding efficient at
runtime~\cite{llguidance}: because constraint-aware training
assumes that the model will be deployed with the analysis
at all times, the analysis must be efficient enough as to
not impose a significant overhead during inference.

\paragraph{Masked Losses.}
Our training objective defined in \Cref{sec:setting} is
not completely novel: variants of the objective
have appeared in different domain-specific settings,
for example GraFT~\cite{graft},
or recommender systems~\cite{reclm}.
As far as we are aware, we are the first to show that the
decomposition holds for any general analysis $A$,
connect it to a semantic objective as above,
and formulate a general mathematical theory on
how constraint-awareness benefits training.
Both GraFT and \citet{reclm} focus on fine tuning,
whereas we investigate the effect of analyses on the pretraining
stage.

%% file: sections/appendix.tex
\section{Proofs for Section~\ref{sec:setting}: Constraint-Aware Loss}
\label{app:objective}

We derive the loss and gradient identities, prove
\Cref{thm:deployed-likelihood}, and connect the token objective to
likelihood over complete strings.

\subsection{Loss and gradient decomposition}
\label{app:derivations}

Fix a prefix $c$ with nonempty allowed set $S=S_A(c)$ and a target
law $\pi$ supported on $S$. Write $p_S$ for the model prediction
renormalized on $S$. Substituting
$p(x\mid c)=p(S\mid c)p_S(x\mid c)$ into cross-entropy gives
\[
 -\sum_{x\in S}\pi(x)\log p(x\mid c)
 =H(\pi)+\KL(\pi\|p_S)-\log p(S\mid c).
\]
Thus the excess CE loss is the sum of the masked KL divergence and
support loss. Taking $\pi=\onehot_y$ gives
\Cref{eq:loss-decomposition}.

Fix a training pair $(c,y)$ with $y\in S_A(c)$, and hold the allowed set
fixed. We use the notation of \Cref{sec:setting}.

To evaluate the support gradient, substitute the softmax probabilities into
$\supportloss=-\log p(S\mid c)$:
\[
 \supportloss
 =\log\sum_{u\in V}e^{z(u\mid c)}
 -\log\sum_{u\in S}e^{z(u\mid c)}.
\]
Differentiating the two logarithms gives
\begin{equation}
 \begin{aligned}
 \frac{\partial\supportloss}{\partial z(x\mid c)}
 &=\frac{e^{z(x\mid c)}}{\sum_{u\in V}e^{z(u\mid c)}}
 -\frac{\mathbf 1\{x\in S\}\,e^{z(x\mid c)}}{\sum_{u\in S}e^{z(u\mid c)}}\\
 &=p(x\mid c)-p_S(x\mid c).
 \end{aligned}
 \label{eq:support-gradient}
\end{equation}
The indicator appears because the second sum contains $z(x\mid c)$ only
when $x\in S$.

For the CA loss, the full-vocabulary normalization cancels in the ratio
$p(y\mid c)/p(S\mid c)$, leaving
\[
 \innerloss=\log\sum_{u\in S}e^{z(u\mid c)}-z(y\mid c).
\]
The first term differentiates to $p_S(x\mid c)$ as above, while the
derivative of $z(y\mid c)$ with respect to $z(x\mid c)$ is $\onehot_y(x)$.
Thus
\begin{equation}
 \frac{\partial\innerloss}{\partial z(x\mid c)}
 =p_S(x\mid c)-\onehot_y(x).
 \label{eq:ca-gradient}
\end{equation}
Substituting these expressions into \Cref{eq:gradient-decomposition}
recovers the ordinary CE gradient:
\[
 \begin{aligned}
 \frac{\partial\celoss}{\partial z(x\mid c)}
 &=\underbrace{p(x\mid c)-p_S(x\mid c)}_{\partial\supportloss/\partial z(x\mid c)}
 +\underbrace{p_S(x\mid c)-\onehot_y(x)}_{\partial\innerloss/\partial z(x\mid c)}\\
 &=p(x\mid c)-\onehot_y(x).
 \end{aligned}
\]

\subsection{Proof of Theorem~\ref{thm:deployed-likelihood}}
\label{app:ca-correctness}

\begin{proof}
Write $D$ for the nonempty training-pair multiset $\corpus$ in the
theorem, and group its pairs by their prefixes. Since the analysis retains
all observed targets, $\widehat p_D(\cdot\mid c)$ is supported on $S_A(c)$,
and the definition of CA loss gives
\begin{align*}
 \innerlossavg(\theta;D)
 &=-\frac1N\sum_{c:n_D(c)>0}n_D(c)\sum_{x\in V}
   \widehat p_D(x\mid c)\log p_{S_A(c),\theta}(x\mid c)\\
 &=\frac1N\sum_{c:n_D(c)>0}n_D(c)
   \left[H(\widehat p_D(\cdot\mid c))
   +\KL\!\left(\widehat p_D(\cdot\mid c)\,\middle\|\,
     p_{S_A(c),\theta}(\cdot\mid c)\right)\right].
\end{align*}
The second equality follows by adding and subtracting the log of the
empirical conditional probability inside the sum. Terms with
$\widehat p_D(x\mid c)=0$ contribute zero.
Subtracting $H_D$ proves \Cref{eq:ca-empirical-correctness}.
This identity holds for every parameter setting $\theta$. Since $H_D$ is
independent of $\theta$, the average CA loss and the weighted average KL
divergence have the same minimizers.
\end{proof}

Each weight $n_D(c)/N$ is positive on the observed prefixes, and each KL
divergence is nonnegative and vanishes exactly when its two distributions
agree. Thus the excess loss $\innerlossavg(\theta;D)-H_D$ is zero precisely
when $p_{S_A(c),\theta}(\cdot\mid c)=\widehat p_D(\cdot\mid c)$ at every
observed prefix. In terms of the raw predictions, this condition is
equivalent to
\[
 p_\theta(x\mid c)=\alpha_c\widehat p_D(x\mid c)
 \quad(x\in S_A(c)),\qquad
 \alpha_c=p_\theta(S_A(c)\mid c)>0.
\]
The remaining mass $1-\alpha_c$ may be distributed among the disallowed
tokens in any proportions. A neural model's parameterization may restrict
which such predictions it can realize. In particular, if an allowed token
has zero empirical frequency, exact agreement requires its masked
probability to be zero. With finite softmax logits this can only be
approached as a limit; the loss identity still holds for every finite
parameter setting.

\subsection{Likelihood of complete strings}
\label{app:ca-strings}

Suppose the training pairs come from $m$ complete strings, each ending
at its first end-of-sequence token $\mathtt{EOS}$. The analysis is
fixed independently of $\theta$ and retains every training target.
For this sequence-level statement, assume the analysis admits a token
at every reachable nonterminal prefix and forces $\mathtt{EOS}$ within
a fixed number of steps. These conditions make sampling from the masked
next-token distributions a normalized distribution over complete strings:
\begin{equation}
 Q^A_\theta(s)=\prod_{t=1}^{|s|}
 p_{S_A(s_{<t}),\theta}(s_t\mid s_{<t}).
 \label{eq:deployed-string-law}
\end{equation}
Here $s_{<t}$ is the prefix before token $s_t$, and the product includes
$\mathtt{EOS}$. At each node of the finite prefix tree, the outgoing
probabilities sum to one, so its terminal leaves have total mass one.
Strings containing a rejected token have probability zero.

For any training string, taking logarithms of the product gives
\begin{equation}
 \sum_{t=1}^{|s|}\innerloss(\theta;s_{<t},s_t)
 =-\log Q^A_\theta(s).
 \label{eq:ca-sequence-likelihood}
\end{equation}
Let $D$ contain all $N$ training pairs from the $m$ strings and let
$\widehat P$ be their empirical string distribution. Summing over the
corpus gives
\[
 \innerlossavg(\theta;D)
 =-\frac1N\sum_{i=1}^m\log Q^A_\theta(s^{(i)})
 =\frac mN\left[H(\widehat P)+\KL(\widehat P\|Q^A_\theta)\right].
\]
The factor $m/N$ and entropy are independent of $\theta$. Minimizing
average CA loss therefore also minimizes the KL divergence from the
empirical string distribution to the deployed generator.

If the masked predictions match the empirical next-token frequencies
at every observed prefix, this generator recovers the empirical string
frequencies exactly. To see this, let $k(c)$ count corpus strings
beginning with $c$, including repeats. Each observed transition has
probability $k(cx)/k(c)$, so the probability of a complete string $s$
telescopes to $k(s)/m$. Unobserved branches have probability zero.
This statement is conditional on fitting those frequencies; it does
not assert that a particular model can attain the entropy bound.

The sequence identity requires scoring every position, including
termination. When only designated answer positions are scored, as in
our experiments, the token identities apply at those positions without
defining a likelihood for the entire surrounding program.

\section{Proofs for Section~\ref{sec:width-theory}: Width}

The width result separates two requirements on the final hidden state:
setting the ratios among allowed tokens, and making those tokens more
likely than all rejected tokens. We first state the full version of
\Cref{thm:width}, including the conditions on a fixed readout. We then
prove the conditions and the matching width bounds. The remaining
subsections explain the readout checks used in the experiments and
how the theorem applies when names span several vocabulary tokens.

\subsection{Capability and proof of Theorem~\ref{thm:width}}
\label{app:width}

\begin{reptheorem}{thm:width}
Let $V$ be a finite vocabulary and $\calS$ a nonempty family of allowed
sets with $\emptyset\ne S\subsetneq V$ for every $S\in\calS$.
Set $m=\max_{S\in\calS}|S|$ and assume $m\ge2$.
A trunk of width $w$ supplies a hidden state $h\in\mathbb R^w$ directly
to a bias-free linear readout $L_R\in\mathbb R^{|V|\times w}$, whose row
for token $v$ is denoted by $L_R[v]$.
Let $p_h=\operatorname{softmax}(L_Rh)$, and let $p_{S,h}$ be $p_h$
restricted to $S$ and renormalized. All matrix and state entries are
finite real numbers.

A fixed $L_R$ is CA-capable on $\calS$ if, for every $S\in\calS$, every
target distribution $\pi$ positive on $S$ and zero outside it, and every
$\epsilon>0$, there exists $h\in\mathbb R^w$ with
$\KL(\pi\|p_{S,h})<\epsilon$. CE capability has the same quantifiers and
requires $\KL(\pi\|p_h)<\epsilon$ instead. The hidden state may depend
on $S$, $\pi$, and $\epsilon$; the readout remains fixed.

\begin{enumerate}
 \item A fixed readout $L_R$ is CA-capable on $\calS$ if and only if,
 for every $S\in\calS$, its rows $\{L_R[v]:v\in S\}$ are affinely
 independent.

 \item A fixed readout $L_R$ is CE-capable on $\calS$ if and only if it
 satisfies the CA condition and, for every $S\in\calS$, there exist a
 vector $g_S\in\mathbb R^w$ and a scalar $\beta_S\in\mathbb R$ such that
 \begin{equation}
  L_R[v]g_S=\beta_S\quad(v\in S),\qquad
  L_R[u]g_S<\beta_S\quad(u\in V\setminus S).
  \label{eq:separator}
 \end{equation}

 \item For any such family $\calS$, a CA-capable readout exists if and
 only if $w\ge m-1$, and a CE-capable readout exists whenever $w\ge2m$.
 If, in addition, there is a set $K\subseteq V$ with $|K|\ge2m+1$ such
 that every $m$-element subset of $K$ belongs to $\calS$, a CE-capable
 readout can exist only if $w\ge2m$. Under this additional assumption,
 the minimum widths admitting the respective capabilities are therefore
 \[
  w^{\mathrm{cap}}_{\mathrm{CA}}=m-1,\qquad
  w^{\mathrm{cap}}_{\CE}=2m.
 \]
\end{enumerate}
The lower bounds constrain the trunk width because it is the input
dimension of $L_R$. The sufficient-width constructions provide readouts
and suitable hidden states; realizing them with a particular trunk
additionally requires that trunk to compute those states.
\end{reptheorem}

For the first two parts, fix an allowed set $S\in\calS$ and write
$O=V\setminus S$. The conditions obtained below must hold for every
$S\in\calS$.

\paragraph{CA criterion: fitting the allowed-token ratios.}
Choose a reference token $v_0\in S$. To fit a positive target $\pi$
after applying the analysis, the hidden state must solve
\[
 (L_R[v]-L_R[v_0])h=\log\frac{\pi(v)}{\pi(v_0)}
 \qquad(v\in S\setminus\{v_0\}).
\]
Every vector in $\mathbb R^{|S|-1}$ occurs on the right for some
positive target. Thus every target can be fitted exactly if and only
if the $|S|-1$ row differences on the left are linearly independent.
Equivalently, the rows indexed by $S$ are affinely independent, or
\[
 \operatorname{rank}[(L_R)_S\mid\mathbf1_S]=|S|.
\]
If this rank condition fails, some target log-ratio vector lies outside
the image of the linear map. The image is closed, so even a sequence
of hidden states cannot approach that vector. Since convergence in
KL to a positive target implies convergence of all its probability
ratios, that target cannot be approximated to arbitrary precision.
This proves the CA criterion and the necessary width $w\ge|S|-1$.

\paragraph{CE criterion: suppressing the rejected tokens.}
Suppose the CA criterion holds, and let $g=g_S$ and $\beta=\beta_S$
satisfy \Cref{eq:separator}. Choose a state $h_0$ fitting any positive target
$\pi$ within $S$. Along $h_\lambda=h_0+\lambda g$, all allowed logits change
by the same amount $\lambda\beta$, so their conditional distribution stays
exactly $\pi$. Every rejected logit falls relative to them as $\lambda\to\infty$.
The support mass tends to one, so total excess CE tends to zero.

\paragraph{Why a successful CE readout must admit such a vector.}
Let $U=\operatorname{col}(L_R)+\mathbb R\mathbf1$, and let $U_S$
be the restriction of this subspace to the allowed coordinates.
If CE excess tends to zero along states $h_n$, both within-mask KL and support
loss tend to zero. Normalize
\[
 z_n=L_R h_n-\log\Big(\sum_{v\in S}e^{L_R[v]h_n}\Big)\mathbf1\in U.
\]
Then $(z_n)_S\to\log\pi$ and every coordinate of $(z_n)_O$ tends to $-\infty$.
The restriction $U\to U_S$ has a linear right inverse $\Gamma$ on its image.
Therefore $v_n=\Gamma((z_n)_S)$ is bounded, and
$d_n=z_n-v_n\in U$ has $(d_n)_S=0$ and $(d_n)_O<0$ for all sufficiently large $n$.
Writing such a $d_n=L_R g-\beta\mathbf1$ gives the required separator.
CE capability also implies CA capability, completing the equivalence.
This argument uses a fixed $L_R$; its right-inverse norm need not be uniform
as the readout changes.

\paragraph{Constructing a CA readout of width $m-1$.}
Place token rows at distinct points on
$t\mapsto(t,t^2,\ldots,t^{m-1})$.
Every set of at most $m$ rows is affinely independent by the Vandermonde
determinant, so width $m-1$ suffices for all allowed sets. A size-$m$ allowed set gives
the matching lower bound.

\paragraph{Constructing a CE readout of width $2m$.}
Use the classical neighborly moment-curve construction~\citep{gale1963neighborly}
in $\mathbb R^{2m}$. 
Every set of at most $m$ rows is
affinely independent. For a particular allowed set $S$ of size $s\le m$, the polynomial
\[
 P_S(t)=-\prod_{v\in S}(t-t_v)^2
\]
has degree $2s\le2m$, vanishes exactly at its token parameters, and is strictly
negative at all other token parameters. Its coefficients give an affine
functional exposing $S$. The CE criterion therefore holds at width $2m$.
Appending zero columns to either construction, with zero entries in the
additional hidden-state coordinates, preserves its predictions and proves
sufficiency at every larger width.

\paragraph{Why smaller CE readouts cannot serve every allowed set.}
Under the additional assumption on $K$, suppose a CE-capable readout exists
with $w<2m$. Each $m$-subset of $K$ is then
the vertex set of an exposed simplex face of the convex hull of all token rows.
Every smaller nonempty subset of $K$ is also an exposed face: extend it to an
$m$-subset and expose the desired subface within that simplex.
To see this directly, add a sufficiently small multiple of a functional exposing
the subface to the functional exposing the full simplex; the strict gaps to
the finitely many other rows persist.

Let the affine dimension of the rows indexed by $K$ be $d\le w$.
As $|K|\ge2m+1\ge d+2$, choose $d+2$ of these rows. An affine dependence,
split into its positive and negative coefficients and normalized, gives disjoint
nonempty sets $A,B$ with intersecting convex hulls. Choose the smaller side as
$A$; then
$|A|\le\lfloor(d+2)/2\rfloor\le m$.
But $A$ is exposed, whereas every row of $B$ lies strictly below its exposing
hyperplane. Their convex hulls cannot intersect, a contradiction.
Thus $w\ge2m$, proving Theorem~\ref{thm:width}.

\paragraph{Relation to the normalized experimental models.}
The theorem allows freely chosen hidden states. Restricting those
states through RMS normalization cannot invalidate its lower bounds.
With unit gain and the positive normalization constant $\nu$ used in the experiments,
\[
 N_\nu(x)=\frac{x}{\sqrt{\|x\|^2/w+\nu}}
 \quad\text{has inverse}\quad
 x(u)=\frac{\sqrt\nu\,u}{\sqrt{1-\|u\|^2/w}}
 \quad\text{for }\|u\|<\sqrt w.
\]
For any finite collection of states $h_c$, choose
$a>\max_c\|h_c\|/\sqrt w$. Replacing $L_R$ by $aL_R$ and supplying
$x(h_c/a)$ before normalization preserves every logit vector:
$(aL_R)N_\nu(x(h_c/a))=L_Rh_c$.
Thus normalization does not obstruct realizing these finite collections
at the stated width when scale may change. It does not make a single
fixed, bounded-scale model capable of arbitrarily small CE excess.

\subsection{Lower bounds for the loss of a learned readout}
\label{app:readout-bounds}

A trained model can have high loss because of its readout, its trunk,
or both. To isolate a restriction imposed by the readout, we fix the
learned $L_R$ and allow the hidden state to be chosen freely for each
target. A positive lower bound in this easier problem also bounds
the loss of the actual trunk. The bound applies to this fixed readout,
not to every readout of the same width.
For a fixed readout and target extended by zero, define
\[
 \rho(L_R,\pi)=\inf_h
 \left[\log\sum_{v\in V}e^{L_R[v]h}
       -\pi^\top L_R h-H(\pi)\right].
\]
Any probability distribution $q$ on $V$ satisfying
$L_R^\top q=L_R^\top\pi$ gives the lower bound
$\rho(L_R,\pi)\ge H(q)-H(\pi)$ by the variational inequality for log-sum-exp.
Requiring $q$ to be supported on $S$ gives the analogous within-mask lower bound.
Feasibility is checked with rational arithmetic.
For a rational target $r$ approximating the intended target $\pi$, an interval
enclosure supplies $\alpha\le\min_{v:r_v>0}\pi_v/r_v$.
A lower bound $L$ at $r$ transfers to
\[
 \rho(L_R,\pi)\ge\max\{0,\ \alpha[L+H(r)]-H(\pi)\}.
\]
Indeed, cross-entropy against $\pi$ is at least $\alpha$ times cross-entropy
against $r$, since $-\log p_v\ge0$. Entropies and the target enclosure are
handled in interval arithmetic. The rank and separator criteria above
establish when the infimum is zero. If no positive lower bound can be
verified, we report only zero as the lower bound.
No optimization or training algorithm can invalidate a
verified bound for the fixed readout, but changing the readout changes the bound.

\subsection{When a variable name spans several vocabulary tokens}
\label{app:tokenization}

The width theorem concerns allowed sets in the model's vocabulary $V$.
The running example and the synthetic scope experiments use one vocabulary
token per variable name. When a name instead consists of several subword
tokens, an analysis over complete names must induce an allowed set at each
intermediate token prefix. For example, if \texttt{x1} and \texttt{x2} are
tokenized as $(\texttt{x},\texttt{1})$ and $(\texttt{x},\texttt{2})$, an
allowed name set $\{\texttt{x1},\texttt{x2}\}$ permits only \texttt{x} at the
first step, then either \texttt{1} or \texttt{2} at the second step.
The CA objective uses these sets of actual next tokens.

To apply \Cref{thm:width}, both $m$ and the condition that every
$m$-element subset of $K$ belongs to $\calS$ must be checked for this
induced family of allowed token sets. A large choice among complete names
can become a sequence of smaller token choices, so neither the largest
allowed-set size nor the condition on $K$ transfers automatically from
lexemes to vocabulary tokens. General implementations must also account
for unfinished lexemes and tokens that cross lexeme boundaries.
The loss decomposition and CA objective still apply to any resulting
computable prefix analysis that retains the desired target tokens.

\section{Proofs for Section~\ref{sec:depth-theory}: Depth}
\label{app:depth}

\Cref{thm:depth} relates the depth needed for small CE excess to the depth
needed to identify the allowed set. We first prove this relation, then
show how CA can leave the decision to the analysis. A separate result
uses a bound on decision accuracy at each input size to obtain an
explicit lower bound on expected CE excess.

\subsection{Analysis accuracy and proof of Theorem~\ref{thm:depth}}
\label{app:depth-definition}

Fix a corpus distribution and a model architecture of width $w$.
Let $\inputsize(c)$ be the size of the input analyzed at prefix $c$;
in the formula example, it counts operators in the current formula.
Each prefix $c$ contains a visible cue $\zeta$ (the tag in the example).
The cue specifies $r\ge2$ nonempty, pairwise disjoint candidate sets
$S^0(\zeta),\ldots,S^{r-1}(\zeta)$ and a target distribution
$\pi^j(\zeta)$ on each set $S^j(\zeta)$.
The analysis selects one of these sets, indexed by $b_A(c)$;
thus $S_A(c)=S^{b_A(c)}(\zeta)$ and $\pi_c=\pi^{b_A(c)}(\zeta)$.

Let $\mathcal F_{L,w}$ be the family of models of depth $L$ and width $w$.
For a model's prediction $p_\theta(\cdot\mid c)$, define its decision by
\[
 \widehat b_\theta(c)
 =\operatorname*{arg\,max}_{0\le j<r}
   p_\theta(S^j(\zeta)\mid c),
\]
with ties broken by the smallest index. For $0\le\alpha\le1$, define
\[
 d_A(\alpha)=\min\left\{L:\exists\theta\in\mathcal F_{L,w},\quad
   \Pr_c[\widehat b_\theta(c)\ne b_A(c)]\le\alpha\right\},
\]
with the minimum of an empty set taken to be $\infty$.
The corpus distribution, candidate sets, width, and
architecture are fixed when comparing depths; $d_A(\alpha)$ suppresses
these fixed choices, as in \Cref{thm:depth}. Error is averaged over
sizes according to their corpus probabilities; it need not
be at most $\alpha$ separately at every size. This definition concerns
the best model at each depth, regardless of how its parameters are obtained.

\begin{proof}[Proof of \Cref{thm:depth}]
For candidate $j$, let $\mu_j=p_\theta(S^j(\zeta)\mid c)$.
The candidate sets are disjoint, so $\sum_j\mu_j\le1$.
On an incorrect decision, some competing set has mass at least
$\mu_{b_A(c)}$. Consequently $\mu_{b_A(c)}\le1/2$, and the support loss
$-\log\mu_{b_A(c)}$ is at least $\log2$.
Support loss is nonnegative on every other prefix. Taking expectations
and using the excess-loss decomposition gives
\begin{equation}
 \E_c\KL\!\left(\pi_c\,\middle\|\,p_\theta(\cdot\mid c)\right)
 \ge\E_c\supportloss(\theta;c)
 \ge(\log2)\Pr_c[\widehat b_\theta(c)\ne b_A(c)].
 \label{eq:depth-error-loss}
\end{equation}
Thus expected CE excess at most $\epsilon$ implies average decision
error at most $\epsilon/\log2$ under the same prefix distribution.
By the definition of
$d_A$, the model's depth must be at least $d_A(\epsilon/\log2)$.
\end{proof}

The same argument applies to a corpus restricted to
$\inputsize(c)\le\maxinputsize$, provided that event has positive
probability. In that case, both the loss expectation and the decision
error defining $d_A$ must use this restricted distribution.
Disjointness is essential: overlapping candidates can share most of
the target mass, so a wrong set decision need not cost $\log2$.
In particular, this reduction does not apply directly to arbitrary
overlapping scope sets.

The depth in $d_A$ counts the transformer layers producing the
probabilities; the fixed rule that sums candidate probabilities and
selects the largest is the same at every depth. This rule uses the
cue-defined candidate sets, not the answer supplied by the analysis.
To apply a recognition lower bound with a different output convention,
one must also account for the computation of these candidate sets and
the sums and comparison. An exact-recognition lower bound alone does
not bound $d_A(\alpha)$ for positive $\alpha$: a model can fail on rare
inputs while attaining high accuracy under the corpus distribution.

\subsection{Leaving the decision to the analysis}

The prediction $p=r^{-1}\sum_j\pi^j(\zeta)$ uses only the visible cue.
After the analysis selects a candidate set, it becomes exactly that
set's target law. It therefore has zero CA excess and support loss
$\log r$, without computing $b_A(c)$.
For a strictly positive softmax prediction on the entire vocabulary,
replace it by $p_\eta=(1-\eta)p+\eta u_V$, where $u_V$ is uniform and
$0<\eta<1$. On the selected set,
\[
 (p_\eta)_{S_A(c)}(v)\ge
 \frac{1-\eta}{1-\eta+r\eta}\,\pi_c(v),
\]
so CA excess tends to zero as $\eta\to0$. Both constructions depend
only on the cue. Realizing them with a shallow transformer requires
that the cue-dependent probabilities be computable at that depth.

\subsection{A loss bound from decision hardness at larger input sizes}
\label{app:depth-hardness}

The next result adds a quantitative assumption about how often models
err at each input size, using the same prefix distribution and
candidate sets as above.
Let $R(L)$ specify a size beyond which a depth-$L$ model's decision
accuracy is bounded as follows.

\begin{theorem}[Depth under decision hardness]
\label{thm:depth-hardness}
Fix $\epsilon\ge0$ and $0\le\delta<1-1/r$.
Suppose that, for every depth $L$, every model
$\theta\in\mathcal F_{L,w}$, and every size $s>R(L)$ with positive
probability under the corpus distribution,
\[
 \Pr\!\left[\widehat b_\theta(c)=b_A(c)\mid \inputsize(c)=s\right]
 \le \frac1r+\delta.
\]
Then every such model has expected CE excess bounded below by
\begin{equation}
 \E_c\KL\!\left(\pi_c\,\middle\|\,p_\theta(\cdot\mid c)\right)
 \ge \E_c\supportloss(\theta;c)
 \ge (1-1/r-\delta)\log2\;\Pr[\inputsize(c)>R(L)].
 \label{eq:depth-bound}
\end{equation}
If a model in $\mathcal F_{L_{\mathrm{in}},w}$ whose predictions depend
only on $\zeta$ attains expected CA excess at most $\epsilon$, then CA
can attain this tolerance at depth $L_{\mathrm{in}}$.
Any CE model attaining the same tolerance must have a depth $L$ for which
the final lower bound in \Cref{eq:depth-bound} is at most $\epsilon$.
\end{theorem}

\begin{proof}
At each size $s>R(L)$, the assumed decision error is at least
$1-1/r-\delta$. Weighting by the corpus probabilities of these sizes gives
\[
 \Pr_c[\widehat b_\theta(c)\ne b_A(c)]
 \ge(1-1/r-\delta)\Pr[\inputsize(c)>R(L)].
\]
Substituting into \Cref{eq:depth-error-loss} proves
\Cref{eq:depth-bound}. A model with CE excess at most $\epsilon$ must
make this lower bound at most $\epsilon$; the stipulated cue-only model
attains that tolerance under CA.
\end{proof}

The decision-hardness assumption concerns every model in the specified
class, on the actual encoded prefix distribution. The trained models'
size-dependent errors in \Cref{sec:depth-theory} do not establish this
assumption or a lower bound on $d_A(\alpha)$ for that distribution.

\section{Proofs for Section~\ref{sec:data}: Data Efficiency}
\label{app:variance}

This appendix makes the training model in \Cref{thm:variance} explicit.
We first describe the two fitting problems, then identify the extra
parameter change required by the support constraints. We follow that
change to an unseen prefix and compare the resulting losses. The last
parts derive the sample requirement and justify the loss decomposition
used to measure corpus variation in \Cref{fig:data-results}.

\subsection{The two fitting problems}
\label{app:variance-fitting}

Let $\corpus$ contain $n$ sampled training prefixes, and write
$\pi_c=\pi_{\text{true}}(\cdot\mid c)$ for the true next-token law.
The fits below use this conditional law at each sampled prefix;
the randomness comes from which prefixes are sampled, rather than
from estimating their target probabilities using next-token counts.

Let $\Delta\in\mathbb R^P$ be a parameter change from a fixed
initialization. At training prefix $c_i$, write $S_i=S_A(c_i)$ and
$O_i=V\setminus S_i$, and choose a reference token $v_i\in S_i$.
The choice is determined by the prefix, independently of a sampled
next token. Let $w_S^0(c_i)$ and $w_O^0(c_i)$ be the initial logits
of the other allowed tokens and the rejected tokens, respectively,
minus the reference logit. In the affine model these become
\[
 w_S(c_i;\Delta)=w_S^0(c_i)+J_S(c_i)\Delta,\qquad
 w_O(c_i;\Delta)=w_O^0(c_i)+J_O(c_i)\Delta,
\]
where the Jacobian matrices $J_S,J_O$ stay fixed during fitting.

Suppose the target law $\pi_{c_i}$ is positive on $S_i$. CA must fit
the $|S_i|-1$ relative logits
$\log[\pi_{c_i}(v)/\pi_{c_i}(v_i)]$ for
$v\in S_i\setminus\{v_i\}$. Subtract the initial values and stack
these equations over training prefixes to obtain
\[
 A_q\Delta=b_q.
\]
There are no equations for rejected tokens in this system.

For the CE fit, also fix a leakage level $\eta\in(0,1)$. If
$k_i=|O_i|>0$, require every rejected relative logit to equal
\[
 \log\frac{\eta}{(1-\eta)k_i\pi_{c_i}(v_i)}.
\]
The allowed exponentials sum to $1/\pi_{c_i}(v_i)$, and the rejected
ones to $\eta/[(1-\eta)\pi_{c_i}(v_i)]$. After normalization, the
rejected mass is therefore exactly $\eta$, while the allowed-token
conditional remains $\pi_{c_i}$. Subtracting the initial rejected
scores and stacking the additional equations gives $A_u\Delta=b_u$.
Prefixes with no rejected tokens add no rows to this system.

The two fitted parameter changes are
\[
 \begin{aligned}
 \Delta_{\mathrm{CA}}&=\arg\min_{A_q\Delta=b_q}\|\Delta\|^2,\\
 \Delta_{\CE}&=\arg\min_{A_q\Delta=b_q,\ A_u\Delta=b_u}\|\Delta\|^2.
 \end{aligned}
\]
We assume the stacked matrix $[A_q^\top\ A_u^\top]^\top$ has full
row rank, so both systems are consistent and their minimum-norm
solutions are unique. There are $n(|V|-1)$ stacked scalar constraints,
so this assumption requires $n(|V|-1)\le P$ and their independence.
The sample-size comparison is restricted to sizes meeting this condition.
This is an idealized interpolation model:
the targets are whole conditional distributions, the CE fit is stopped
at finite leakage, and the Jacobians do not move. A sampled-token
AdamW run need not satisfy these conditions. Expectations in the
theorem are over samples for which these rank and moment assumptions
hold; duplicate or dependent constraints need not meet the rank condition.

\subsection{The extra parameter change needed to fit support}
\label{app:variance-change}
Let $K=\ker A_q$ and $P_K$ be its orthogonal projector; a superscript
$+$ denotes the Moore--Penrose pseudoinverse.
The minimum-norm CA solution is $\Delta_{\mathrm{CA}}=A_q^+b_q$ and lies in
$\operatorname{row}(A_q)=K^\perp$.
Every solution of both systems is therefore uniquely
$\Delta_{\mathrm{CA}}+v$ with $v\in K$ and
$A_uv=b_u-A_u\Delta_{\mathrm{CA}}$.
By orthogonality its squared norm is
$\|\Delta_{\mathrm{CA}}\|^2+\|v\|^2$. Hence the minimum-norm addition is
\[
 \Delta_A=(A_uP_K)^+(b_u-A_u\Delta_{\mathrm{CA}}).
\]
The stacked rank assumption makes $A_uP_K$ full row rank: any row dependence
would put a combination of rows of $A_u$ in $\operatorname{row}(A_q)$.
This proves the decomposition and orthogonality in Theorem~\ref{thm:variance}.
At held-out $c$, the affine model gives the exact difference
$\delta_c=J_S(c)\Delta_A$.

\subsection{Full statement and proof of Theorem~\ref{thm:variance}}
\label{app:variance-loss}
At a held-out prefix, the CA relative logits are
$w_c=w_S^0(c)+J_S(c)\Delta_{\mathrm{CA}}$, and CE changes them by
$\delta_c=J_S(c)\Delta_A$. The expected relative-logit loss is
$f_c(w)=\log(1+\sum_j e^{w_j})-\sum_j\pi_jw_j$,
where the sum omits the reference token.
Write $q_{\mathrm{CA},\corpus}(\cdot\mid c)=\operatorname{softmax}(0,w_c)$
and $q_{\CE,\corpus}(\cdot\mid c)=\operatorname{softmax}(0,w_c+\delta_c)$
for the fitted models' masked predictions. For $X\in\{\mathrm{CA},\CE\}$,
define
\[
 R_X(n)=\E_{\corpus,c}\KL\!\left(
 \pi_{\text{true}}(\cdot\mid c)\,\middle\|\,
 q_{X,\corpus}(\cdot\mid c)\right),
\]
where $\corpus$ has size $n$ and $c$ is drawn independently from the
held-out prefix distribution. This distribution and the law used to
sample training corpora are the same for both fits.

\begin{reptheorem}{thm:variance}
Under the fitting assumptions in \Cref{app:variance-fitting},
$\Delta_{\CE}=\Delta_{\mathrm{CA}}+\Delta_A$, where
$\Delta_A\perp\Delta_{\mathrm{CA}}$ and $A_q\Delta_A=0$.
Thus the additional change preserves every fitted training conditional.
Suppose its effect at held-out prefixes is conditionally centered:
\[
 \E[\delta_c\mid c,\Delta_{\mathrm{CA}}]=0.
\]
If the expectations below are finite, then
\begin{equation}
 \begin{aligned}
 0\le R_{\CE}(n)-R_{\mathrm{CA}}(n)
 &=\E[f_c(w_c+\delta_c)-f_c(w_c)]\\
 &\le\tfrac12\E\|\delta_c\|^2.
 \end{aligned}
 \label{eq:variance-bound}
\end{equation}
In particular, $R_{\CE}(n)>R_{\mathrm{CA}}(n)$ if
$\Pr_{\corpus,c}[\delta_c\ne0]>0$.
Consequently, on any common range of sample sizes satisfying these
assumptions, CA requires no more examples than CE to reach a fixed
expected within-set excess loss.
\end{reptheorem}

The parameter decomposition was proved in \Cref{app:variance-change}.
The equality in \Cref{eq:variance-bound} follows by subtracting the
same target entropy from the two expected token losses.
To compare those losses, recall the function $f_c$ above.
It is convex, with Hessian
$H_c=\operatorname{diag}(q_-)-q_-q_-^\top$, where
$q=\operatorname{softmax}(0,w)$ and $q_-$ omits the reference token.
For any $a$, $a^\top H_ca$ is a variance under $q$ and lies between zero and
$\|a\|^2$, so $0\preceq H_c\preceq I$.
Fix $c$ and $\Delta_{\mathrm{CA}}$, which also fixes $w_c$.
The centering assumption says that the mean of $w_c+\delta_c$ under
this conditioning is $w_c$. Jensen's inequality for the convex
function $f_c$ therefore gives
\[
 \E[f_c(w_c+\delta_c)\mid c,\Delta_{\mathrm{CA}}]\ge f_c(w_c).
\]
Averaging once more proves the lower bound in
\Cref{eq:variance-bound}. For the upper bound, the Hessian bound gives
\[
 f_c(w+\delta_c)\le f_c(w)+\nabla f_c(w)^\top\delta_c
                         +\tfrac12\|\delta_c\|^2
\]
and conditional centering removes the linear term in expectation, giving the
upper bound.

For finite logits, $f_c$ is strictly convex in the nonreference
coordinates: its Hessian is positive definite because every token,
including the reference, has positive probability.
Equality in the expected-loss comparison therefore holds exactly when
$\delta_c=0$ almost surely, meaning the extra CE update leaves the
held-out masked predictions unchanged.

Centering also identifies the added variance directly. When the relative
logits have finite second moments, set $\bar w_c=\E[w_c\mid c]$.
Both fits have this same mean relative-logit vector at $c$, and
\[
 \E[\|w_c+\delta_c-\bar w_c\|^2\mid c]
 =\E[\|w_c-\bar w_c\|^2\mid c]+\E[\|\delta_c\|^2\mid c].
\]
The cross term vanishes because $w_c$ is fixed given
$c,\Delta_{\mathrm{CA}}$ and $\delta_c$ is centered under that conditioning.
This is a variance comparison for relative logits; it does not require
the mean CA prediction to equal the true conditional distribution.

Conditional centering is an additional assumption: preserving the
training conditionals does not prevent a systematic change at an
unseen prefix. The loss ordering above uses this assumption to remove
the linear term; the larger parameter norm alone does not imply it.

\subsection{How many examples are needed to reach a target loss?}
\label{app:sample-requirement}
Let $R_X(n)$ be the expected held-out within-mask excess of learner $X$,
including expectation over its training sample. Consider a common set
$\mathcal I$ of sample sizes at which the two learner families satisfy
Theorem~\ref{thm:variance}, and define
\[
 N_X(\epsilon)=\min\{n\in\mathcal I:R_X(n)\le\epsilon\},
 \qquad \min\varnothing=\infty.
\]
The theorem gives $R_{\mathrm{CA}}(n)\le R_{\CE}(n)$ at each such $n$.
Hence every sample size meeting the CE target also meets the CA target, so
\[
 N_{\mathrm{CA}}(\epsilon)\le N_{\CE}(\epsilon).
\]
If $n_*=N_{\mathrm{CA}}(\epsilon)<\infty$ and
$R_{\CE}(n_*)>\epsilon$, the inequality is strict: CE meets the target neither
at $n_*$ nor at any earlier size, where even CA fails to meet it.
No monotonic learning curve or asymptotic rate is needed for this set-inclusion
argument. A positive risk gap alone does not ensure strictness at every
chosen tolerance, because both risks can already be below it.
The result concerns the specified learner families and valid sample-size
range; it supplies neither a universal sample multiplier nor a rate for AdamW.

\subsection{The corpus-fluctuation loss in Figure~\ref{fig:data-results}}
\label{app:ensemble-fluctuation}

Fix one prefix with target law $\pi$, and let $q_1,\ldots,q_{16}$ be
the masked predictions from the sixteen training corpora. Their mean
is $\bar q=\frac1{16}\sum_iq_i$. Subtracting its loss from the average
individual loss gives
\begin{equation}
 \frac1{16}\sum_{i=1}^{16}\KL(\pi\|q_i)-\KL(\pi\|\bar q)
 =\sum_{v\in S}\pi(v)
   \left[\log\bar q(v)-\frac1{16}\sum_{i=1}^{16}\log q_i(v)\right]
 \ge0.
 \label{eq:ensemble-identity}
\end{equation}
Concavity of the logarithm makes every bracket nonnegative.
Averaging this identity over the shared evaluation prefixes gives
$\fluctuationloss$ in \Cref{eq:prediction-variance}; the average
individual excess is exactly
$\meanpredictionloss+\fluctuationloss$.
Both components are measured in nats. The latter is the average loss
removed by averaging predictions, rather than the variance of the
scalar losses. The identity needs no assumptions about training or
conditional centering.

\section{Experimental Setup and Supporting Results}
\label{app:experiments}

All three studies use known target distributions so that excess loss
can be measured directly. The width study varies model size, the depth
study varies the number of layers, and the data-efficiency study varies
the available training corpus. We give their shared architecture and
optimizer below, followed by the setup and supporting results for each.

\paragraph{Shared architecture and optimization.}
The model is a causal decoder with bias-free attention and MLP projections,
rotary position embeddings (base 10,000), pre-block RMS normalization,
SwiGLU MLPs, and an untied output readout. An optional linear bottleneck projects
the trunk to the readout dimension before final RMS normalization.
The MLP hidden dimension is $8\lceil8d/24\rceil$ for trunk dimension $d$.
Vocabulary size is 134, context length 4,096, and there are four heads.
Weights use normal initialization with standard deviation 0.02; RMS
normalization uses $\epsilon=10^{-5}$.
AdamW uses $(\beta_1,\beta_2)=(0.9,0.95)$, $\epsilon=10^{-8}$, weight decay 0.1,
and gradient clipping at norm 1, with bf16 forward/backward autocast.
Each update processes 16 documents. The nominal $6\times10^8$-position budget
rounds to 9,156 complete updates, or 600,047,616 processed positions.
Only designated answer positions have positive training weight.

The scope width and main data recipes use peak learning rate $3\times10^{-3}$,
2\% warmup (183 updates), and a cosine schedule.
The width runs end at 10\% of peak learning rate; the
16-corpus data study ends at zero.
Parameters and samples are matched within each objective pair.
Changing readout width changes both the bottleneck and output parameter counts;
the capacity theorem refers specifically to the readout dimension.

\input{sections/appendix-width-experiments}
\input{sections/appendix-depth-experiments}
\input{sections/appendix-data-experiments}

%% file: sections/appendix-width-experiments.tex
\subsection{Width: setup and controls for Section~\ref{sec:width-theory}}
\label{app:width-experiments}

\paragraph{Target distribution.}
The scope analysis reads the twelve listed variable names. The
remaining prediction task uses a fixed table $Z$, with one score
for each identifier and function. At a scored prefix, its target is
the softmax of the selected function's column restricted to the
listed scope. We construct this table so that eleven score features
are sufficient for every function.
For identifier $i=0,\ldots,60$, let $t_i=-1+i/30$.
Orthonormalize the columns of the Vandermonde matrix
$(1,t,\ldots,t^{11})$ in order and put $A=\sqrt{61}\,Q_{\cdot,1:11}$.
Let $U\in\mathbb R^{11\times12}$ have the unit vertices of a regular simplex
as columns (a scaled Helmert contrast matrix), and set $Z=0.8AU$.
Then $\operatorname{rank}Z=11$, its columns are centered, and every 12-row
restriction of $[\mathbf1,A]$ is nonsingular. Since the 12 simplex columns span
$\mathbb R^{11}$, the conditional targets at each allowed set span its
relative-logit space. The cue-only state $U_{\cdot f}$ and readout rows
$0.8A_{i\cdot}$ fit the ideal masked law exactly. Equal state norms allow
the final RMS gain to absorb the common normalization.
The float64 law differs by at most $5.4\times10^{-13}$ per logit from
the ideal construction; all reported KL values use that numerical law.
Statements are independent: each draws a uniform 12-subset of names,
a uniform function, and then an answer from the corresponding target
law, with 240 queries per document. The target entropy is approximately 2.2052
nats, and uniform prediction within the allowed set has excess 0.2797.

\paragraph{Training and evaluation.}
The sweep in \Cref{fig:width-results} uses four layers of internal
width 64, with a learned projection to readout widths
$8,10,11,12,16,23,24,32,64$.
The training pool contains 150,000 documents; 9,156 updates consume
146,496 of them. Each document packs 240 training pairs, giving
35,159,040 pairs consumed under the shared recipe above.
Excess losses are evaluated against the known target law on 480,000
queries from 2,000 held-out documents.

\paragraph{What the learned readout can represent.}
To distinguish a readout limitation from an imperfect trunk, we hold
$L_R$ fixed and allow a different hidden state for every target.
The rank and separator conditions in \Cref{app:width} determine
capability; the bounds in \Cref{app:readout-bounds} measure residual
loss when arbitrary precision is impossible for that readout.
The latter use 256 sampled allowed sets and all twelve functions,
giving 3,072 pairs, with a set-bootstrap 5--95\% interval.
This sample is separate from the 480,000-query loss panel.

The CA readout at width 11 can fit all 3,072 checked targets exactly.
At the same width, the CE readout has mean excess-loss lower bound
1.0172 on this panel. Its actual excess on this panel is 1.0786;
the 1.0698 reported on the main validation panel uses different queries.
At width 24, the learned CE readout still has a positive mean lower
bound of approximately 0.153. Width 32 is the first tested CE width
at which every checked target can be fitted to arbitrary precision.
These checks apply to the learned readouts on the sampled targets;
they do not say that every readout of the same width has the same loss.

\paragraph{Fixed-scope control.}
Fixing the same twelve in-scope names throughout the corpus, with
their order randomized, lets both objectives fit accurately at readout
width 12: CE total excess is 0.000229 and CA masked excess is 0.000200.
The learned CE readout satisfies the rank and separator conditions
for this single allowed set. This control changes both the input
distribution and the task of identifying the current scope.

\paragraph{Longer-schedule control.}
We keep trunk and readout width 64, four layers, seed 0, the
150,000-document pool, and the same 480,000-query evaluation panel.
Updates increase from 9,156 to 18,311, with the full cosine schedule
rescaled and warmup increased from 183 to 366 updates. The longer
run consumes all 150,000 documents before repeating most of them,
versus 146,496 distinct documents in the shorter run.
CE's masked excess decreases from 0.027137 to 0.015593 and its total
excess from 0.0537 to 0.0290. CA's masked excess decreases from
0.000520 to 0.000472. Thus CE improves under the longer recipe but
retains higher loss. The comparison changes exposure and the learning
rate trajectory together; it does not establish a fixed optimization
floor. Each objective has one run under each schedule.

%% file: sections/appendix-depth-experiments.tex
\subsection{Depth: setup and controls for Section~\ref{sec:depth-theory}}
\label{app:depth-experiments}

The depth comparison follows one task through the same curriculum for
both objectives. CA learns the tag-dependent probabilities at the
smallest tested depth. Additional CE layers improve formula evaluation,
reducing support loss, while its masked prediction remains less accurate.

\paragraph{Formula distribution and matched curriculum.}
Formula tree shape is proposed uniformly over full binary trees of a
given operator count, and rejection sampling balances the five values
exactly at each size. This accepted distribution is the data law.
Both objectives use width 128, depths 2, 4, and 8, and seeds 0 and 1.
Each model starts from random initialization and first trains on the
uniform short-size law $s\in\{2,3,4\}$, then continues on
$\Pr(s)\propto0.85^{s-2}$ for $2\le s\le24$.
Each stage processes approximately $6\times10^8$ positions in 9,156
updates. The first stage uses peak learning rate $3\times10^{-3}$.
Both stages use 2\% warmup and a cosine schedule ending at 10\% of
the peak rate. Continuation uses peak learning rate $10^{-3}$, a fresh
optimizer state, and a restarted schedule. CA continues from its own
CA-trained first stage. Each objective pair shares both corpora,
initialization and ordering seeds, and both stages' recipes.
Separate hyperparameter selection for CA and CE remains untested on
this task.

\paragraph{Loss and decision accuracy.}
The primary evaluation samples operator counts uniformly from 2 to 24
on 8,000 documents, giving 551,576 queries and at least 23,400 queries
per size. The loss columns in \Cref{tab:depth-pairs} are query means;
the main figure averages the two seeds at each depth.

\input{figures/depth-pairs-table}

CA's masked excess is lower in every pair, by factors of 54--91.
Within each trained-model pair, the nominal 95\% intervals from 10,000
document-clustered bootstrap replicates are disjoint. For example,
depth 2, seed 0 gives $[3.004,3.013]\times10^{-5}$ for CA and
$[1.631,1.647]\times10^{-3}$ for CE. These intervals describe variation
over held-out documents; the two-seed ranges are not confidence
intervals over training runs.

CE's total excess is over 99.6\% support loss in each pair, and its
accuracy at identifying the allowed group improves with depth.
The right panel of \Cref{fig:depth-results} averages query-level CE excess
losses within each size and then averages the two seeds.
For decision accuracy, the largest difference between document-level
and query-weighted estimates over the 96 size/objective/seed cells
at sizes 2--9 is 0.00313. CA's decision accuracy remains
near $1/5$, while the analysis supplies the correct group at deployment.
\input{figures/depth-size-results-figure}

The reported 95\%-accuracy range requires every size up to its endpoint
to have a nominal Bonferroni document-bootstrap lower bound of at least
95\%, using 10,000 replicates and level $0.05/23$ per size.
All-correct sizes use the exact document-binomial bound
$(0.05/23)^{1/n_{\mathrm{docs}}}$ instead. A size needs at least 30
contributing documents; a missing size or a degenerate non-all-correct
bootstrap stops the range. The CE endpoints are $(4,5,8)$ across
depths $(2,4,8)$ at seed 0 and $(3,5,7)$ at seed 1. These are properties
of the trained models, rather than hardness bounds for the model class.

\paragraph{Evaluation under the training size distribution.}
Under the geometric size law used during the second training stage,
the ordering of CE total loss by depth is unchanged. On 2,000
held-out documents and 230,537 queries, the total excesses at depths
2, 4, and 8 are $(0.6672,0.4183,0.1618)$ for seed 0 and
$(0.7213,0.4349,0.2239)$ for seed 1. CA's masked losses remain in the
same small range as under uniform sizes. This checks the effect of
the evaluation size distribution; neither evaluation establishes a
lower bound on the best achievable error at a given depth.

%% file: figures/depth-pairs-table.tex
\begin{table}[!t]
 \centering\small
 \caption{All six matched two-stage pairs, evaluated on the same uniform-size
 panel. Losses are nats per answer; the two within-mask columns use different
 scales. The last column is the largest operator count up to which CE
 meets the 95\%-accuracy criterion at every size. No CA model meets this
 criterion even at the first evaluated size.}
 \label{tab:depth-pairs}
 \input{figures/depth-table}
\end{table}

%% file: figures/depth-table.tex
\begin{tabular}{rrcccc}
\toprule
Depth & Seed & $\innerloss$ ($10^{-5}$) & $\celossmask$ ($10^{-3}$) & $\celoss$ & CE range\\
\midrule
2 & 0 & 3.008 & 1.639 & 1.1036 & 4\\
2 & 1 & 2.382 & 1.536 & 1.1006 & 3\\
4 & 0 & 3.283 & 1.972 & 0.8318 & 5\\
4 & 1 & 2.451 & 2.236 & 0.8552 & 5\\
8 & 0 & 3.178 & 1.706 & 0.4764 & 8\\
8 & 1 & 2.450 & 1.494 & 0.5861 & 7\\
\bottomrule
\end{tabular}

%% file: figures/depth-size-results-figure.tex
\begin{figure}[!t]
 \centering
 \input{figures/depth-size-results}
 \caption{Full-size CE curves for the same two-stage seed-0 models. Left:
 total excess. Right: query-weighted accuracy of identifying the allowed
 set. The main figure averages the two seeds and shows loss by size only
 for operator counts 2--9.}
 \label{fig:depth-full}
\end{figure}
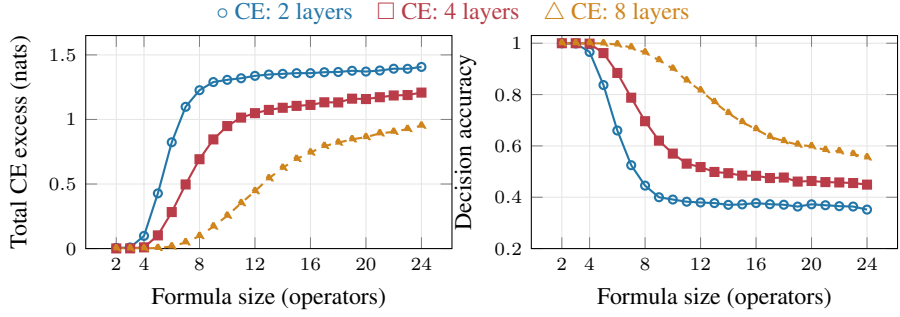

%% file: figures/depth-size-results.tex
\definecolor{CAblue}{HTML}{2376A8}
\definecolor{CEred}{HTML}{B93C48}
\definecolor{CEorange}{HTML}{CF841C}
{\small\textcolor{CAblue}{$\circ$ CE: 2 layers}\quad
\textcolor{CEred}{$\square$ CE: 4 layers}\quad
\textcolor{CEorange}{$\triangle$ CE: 8 layers}}\par
\begin{tikzpicture}
\begin{groupplot}[group style={group size=2 by 1,horizontal sep=1.05cm},
width=0.46\textwidth,height=4.4cm,xlabel={Formula size (operators)},
xtick={2,4,8,12,16,20,24},
tick label style={font=\scriptsize},label style={font=\small},
grid=major,grid style={gray!20},/tikz/mark size=1.5pt,scaled ticks=false,minor tick num=0]
\nextgroupplot[ylabel={Total CE excess (nats)},ymin=0,ymax=1.65]
\addplot[CAblue,mark=o,thick] coordinates {(2,0.00549874431693) (3,0.00849441333398) (4,0.0984166919629) (5,0.427543084512) (6,0.823386189036) (7,1.09798153007) (8,1.22661898291) (9,1.28968736239) (10,1.30688672358) (11,1.3191870446) (12,1.3377610385) (13,1.34795979262) (14,1.35337116742) (15,1.35878119555) (16,1.35866165491) (17,1.36594198864) (18,1.36780041279) (19,1.37719334991) (20,1.37100404482) (21,1.37916063823) (22,1.39301943981) (23,1.39243365833) (24,1.40693728069)};
\addplot[CEred,mark=square*,thick] coordinates {(2,0.00118117321652) (3,0.00172630986416) (4,0.0094603582245) (5,0.102257322981) (6,0.282738949577) (7,0.496694379151) (8,0.692134482614) (9,0.845556831287) (10,0.948849273004) (11,1.01448080635) (12,1.04879056133) (13,1.07394578852) (14,1.09051942152) (15,1.10509543435) (16,1.1126332567) (17,1.13291273847) (18,1.13214580291) (19,1.16107229549) (20,1.15810271534) (21,1.1723487853) (22,1.18568649319) (23,1.18908823497) (24,1.20804257571)};
\addplot[CEorange,mark=triangle*,dashed,thick] coordinates {(2,0.00105486364241) (3,0.00108313537878) (4,0.00134742249101) (5,0.00450636189573) (6,0.0158150176192) (7,0.0480990151306) (8,0.0970458469211) (9,0.170657357865) (10,0.254351793967) (11,0.351890306452) (12,0.444763933616) (13,0.543679075144) (14,0.625318427835) (15,0.694488143096) (16,0.743530538404) (17,0.794226779616) (18,0.820229761288) (19,0.846069579361) (20,0.863191437571) (21,0.892036325288) (22,0.904197229941) (23,0.92569499465) (24,0.951728826233)};
\nextgroupplot[ylabel={Decision accuracy},ymin=0.2,ymax=1.03]
\addplot[CAblue,mark=o,thick] coordinates {(2,0.999629171817) (3,0.99892526455) (4,0.966015542953) (5,0.837081210916) (6,0.65998752858) (7,0.524814722292) (8,0.445259989279) (9,0.400025009379) (10,0.391094021335) (11,0.382596742617) (12,0.379562955255) (13,0.377222569228) (14,0.370342616877) (15,0.372160264354) (16,0.377297114613) (17,0.372835157462) (18,0.370821304548) (19,0.363632573073) (20,0.37251272559) (21,0.368850401606) (22,0.365281948615) (23,0.363772135694) (24,0.351872776732)};
\addplot[CEred,mark=square*,thick] coordinates {(2,1) (3,1) (4,0.998361688721) (5,0.961610137905) (6,0.884140511328) (7,0.787784161879) (8,0.696259948043) (9,0.620440998708) (10,0.569999173075) (11,0.530886824676) (12,0.517585848075) (13,0.498563397876) (14,0.493734648849) (15,0.484427922346) (16,0.483503020574) (17,0.474818635468) (18,0.476800133886) (19,0.461034899721) (20,0.463337680367) (21,0.459044511379) (22,0.457624290958) (23,0.455344570087) (24,0.449173467253)};
\addplot[CEorange,mark=triangle*,dashed,thick] coordinates {(2,1) (3,1) (4,1) (5,0.999047500725) (6,0.995094574932) (7,0.983470730286) (8,0.963630365758) (9,0.933850193823) (10,0.900562308774) (11,0.854667388678) (12,0.81581685744) (13,0.770560066625) (14,0.72782148953) (15,0.692523750516) (16,0.665075408728) (17,0.635341971736) (18,0.618635203548) (19,0.60480340241) (20,0.598376172647) (21,0.583960843373) (22,0.578953953954) (23,0.568209942394) (24,0.554969658924)};
\end{groupplot}
\end{tikzpicture}

%% file: sections/appendix-data-experiments.tex
\subsection{Data efficiency: setup and controls for Section~\ref{sec:data}}
\label{app:data-experiments}

\paragraph{Training corpora and evaluation.}
We use the scope law in \Cref{app:width-experiments}, four layers,
internal width 128, and readout width 64. Each 4,096-token document
packs 240 training pairs. The 2.4M, 7.2M, 24M, and 36M pairs in the
main text therefore correspond to pools of 10k, 30k, 100k, and 150k
documents. At each size, sixteen independently generated corpora
provide the training replications. Their generator seeds are
$2000+10k+j$, with corpus index $k=1,\ldots,16$ and size index
$j=1,\ldots,4$.

Each CA/CE pair shares its corpus, initialization seed 0, batch-order
seed 0, and 9,156-update schedule ending at zero learning rate.
Smaller pools are revisited; the largest contributes 146,496 distinct
documents within the fixed budget. All 128 runs are evaluated at the
final checkpoint on the same 120,000 queries from 500 held-out
validation documents. The training corpora are the independent
replications; the evaluation queries are shared across models.

\paragraph{Loss and corpus variation.}
At each held-out prefix, we normalize each model's prediction over
its allowed set and average these distributions across the sixteen
corpora. We then compute the KL loss of this mean prediction and
subtract it from the mean individual KL loss, exactly as in
\Cref{eq:prediction-variance}. Finally, we average both quantities over
the shared evaluation panel. \Cref{app:ensemble-fluctuation} proves
that the difference is nonnegative. Both quantities are losses in
nats, rather than squared distances between probability vectors.

\input{figures/data-detail-table}

\Cref{tab:data-detail} adds the across-corpus sample standard deviations
and the CE support loss to the main figure. Across the four sizes,
averaging the sixteen predictions removes 89--93\% of CA's excess
and 91--93\% of CE's. The absolute amount removed is 75--146 times
larger for CE. CA at 2.4M pairs has lower masked excess and
$\fluctuationloss$ than CE at 36M pairs under this fixed recipe.
These comparisons measure corpus-dependent prediction error; they
do not establish the minimum-change or conditional-centering
assumptions of \Cref{thm:variance} for AdamW training.

\paragraph{Independent hyperparameter selection.}
At each of 2.4M and 36M training pairs, each objective searches learning
rates $\{0.001,0.003,0.01\}$ and readout weight decays $\{0,0.1\}$,
using two independent tuning corpora per recipe. Architecture,
initialization and order seed, and the zero-endpoint schedule match
the primary study. Each objective selects by mean masked loss on the
tuning corpora's validation splits. The choices are fixed before
training on four new evaluation corpora per size; the common
120,000-query reporting panel is held out from selection.

\input{figures/tuning-results-table}

CA has lower masked excess in all eight evaluation pairs
(\Cref{tab:independent-tuning}). The ratios of mean CE to mean CA
excess are approximately 193 at 2.4M pairs and 156 at 36M.
Selecting CE by unmasked validation loss instead gives rate 0.001
and no readout decay at 2.4M, with masked excess 0.027084 versus
0.029823 under masked selection. The paired difference has unresolved
sign in the four-corpus resampling analysis. At 36M both criteria
select the same CE recipe. CA remains lower in every pair under
either selection criterion. These comparisons assess sensitivity
within the stated grid, rather than optimal hyperparameters.

\paragraph{Longer training with fixed and larger data pools.}
Three cells, each with sixteen paired corpus runs, distinguish
repeating a fixed pool from accessing a larger one under the same
longer schedule:
\begin{center}
 \small
 \begin{tabular}{lrrr}
 \toprule
 Cell & Available pairs (M) & Updates & Distinct pairs consumed (M)\\
 \midrule
 A & 36 & 9,156 & 35.15904\\
 B & 36 & 18,312 & 36\\
 C & 72 & 18,312 & 70.31808\\
 \bottomrule
 \end{tabular}
\end{center}
Cell A reuses the primary runs; B and C add 64 runs with the same
architecture, initialization and order seeds, and evaluation panel.
A and B use identical available corpora. B rescales the entire cosine
schedule, including warmup from 183 to 366 updates. Each C corpus
extends its A generator stream, but its larger pool changes the
loader's permutation. Thus C is not a continuation of A's realized
training stream. The missing 72M-pair short-schedule cell prevents
separating all effects of corpus size and schedule. The longer
schedule here has 18,312 updates, versus 18,311 in the width control.

\input{figures/schedule-loss-table}

CA has smaller mean masked excess and $\fluctuationloss$ in all three
cells (\Cref{tab:budget-loss}). Its fluctuation loss decreases from
$6.64\times10^{-5}$ to $4.14\times10^{-5}$ nats with the longer
schedule at the same available pool size, showing that the residual
in the primary sweep is sensitive to the training recipe.
CE's loss varies substantially across corpora: its standard deviation
is 1.10 times its mean in B and 1.13 times its mean in C, versus
0.37 in A. Resampling the sixteen corpus indices jointly across cells
and objectives, using 2,000 replicates, gives 2.5--97.5 percentile
bands spanning zero for the CE changes in log fluctuation loss.
These bands describe sensitivity to the observed corpora. These controls support the observed
CA advantage while leaving the source of the residual errors open.

%% file: figures/data-detail-table.tex
\begin{table}[!t]
 \centering\small
 \setlength{\tabcolsep}{3pt}
 \caption{Primary data study over sixteen independent training corpora
 at each pool size, with fixed initialization and order seed. Masked
 excess is mean $\pm$ sample standard deviation; other entries are means.
 CE masked excess plus support loss equals CE total before rounding.
 Pool sizes are in millions of training pairs; losses are held-out nats per query.}
 \label{tab:data-detail}
 \begin{tabular}{rcccc}
 \toprule
 Pairs (M) & $\innerloss$ & $\celossmask$ & CE $\supportloss$ & $\celoss$\\
 \midrule
 2.4 & $0.000379\pm0.000080$ & $0.028276\pm0.002998$ & 0.017490 & 0.045766\\
 7.2 & $0.000080\pm0.000005$ & $0.011467\pm0.005526$ & 0.017969 & 0.029436\\
 24 & $0.000075\pm0.000004$ & $0.007973\pm0.003928$ & 0.015902 & 0.023875\\
 36 & $0.000074\pm0.000003$ & $0.007268\pm0.002713$ & 0.015219 & 0.022487\\
 \bottomrule
 \end{tabular}
\end{table}

%% file: figures/tuning-results-table.tex
\begin{table}[!t]
\centering\small
\caption{Independent selection for masked deployment. Each objective uses
the same six-recipe search and two tuning corpora per size. Losses are
means over four new evaluation corpora; the masked-loss column also gives
the sample standard deviation. Pool sizes are in millions of training
pairs; all losses are excess in nats per query.
The CE total column describes the same masked-loss-selected CE model.}
\label{tab:independent-tuning}
\begin{tabular}{llrrcc}
\toprule
Pairs (M) & Objective & Peak LR & Readout decay & Masked excess & CE total\\
\midrule
2.4 & CA & 0.010 & 0 & $0.000155\pm0.000013$ & --\\
2.4 & CE & 0.003 & 0.1 & $0.029823\pm0.005366$ & 0.049277\\
36 & CA & 0.001 & 0 & $0.000060\pm0.000004$ & --\\
36 & CE & 0.003 & 0.1 & $0.009372\pm0.002489$ & 0.027421\\
\bottomrule
\end{tabular}
\end{table}

%% file: figures/schedule-loss-table.tex
\begin{table}[!t]
 \centering\small
 \caption{Larger-pool and longer-schedule study. Within-mask excess is
 mean $\pm$ sample standard deviation over 16 corpora.
 $\fluctuationloss$ is the average loss removed by averaging their masked
 predictions. Both numerical columns use units of $10^{-3}$ nats.}
 \label{tab:budget-loss}
 \input{figures/budget-table}
\end{table}

%% file: figures/budget-table.tex
\begin{tabular}{llcc}
\toprule
Cell & Objective & Masked excess ($10^{-3}$ nats) & $\fluctuationloss$ ($10^{-3}$ nats)\\
\midrule
A & CA & $0.074\pm0.003$ & 0.0664\\
A & CE & $7.268\pm2.713$ & 6.5800\\
B & CA & $0.047\pm0.002$ & 0.0414\\
B & CE & $4.899\pm5.365$ & 4.4032\\
C & CA & $0.046\pm0.002$ & 0.0410\\
C & CE & $5.947\pm6.708$ & 5.3315\\
\bottomrule
\end{tabular}